%% file: main.tex
\documentclass[11pt]{article}
\usepackage[letterpaper,margin=1in]{geometry}
\usepackage{amsmath}
\usepackage{amsthm}
\usepackage{ifpdf}
\usepackage[hidelinks]{hyperref}
\usepackage{algorithm}
\usepackage[capitalize,nameinlink]{cleveref}

\numberwithin{equation}{section}
\numberwithin{algorithm}{section}

\theoremstyle{plain}
\newtheorem{theorem}{Theorem}[section]
\newtheorem{lemma}[theorem]{Lemma}
\newtheorem{corollary}[theorem]{Corollary}
\newtheorem{proposition}[theorem]{Proposition}
\newtheorem{definition}[theorem]{Definition}
\newcommand{\newsiamthm}[2]{\theoremstyle{plain}\newtheorem{#1}[theorem]{#2}}
\newcommand{\newsiamremark}[2]{\theoremstyle{remark}\newtheorem{#1}[theorem]{#2}}

\DeclareUrlCommand\emailaddress{\urlstyle{rm}}
\newcommand{\email}[1]{\protect\href{mailto:#1}{\protect\emailaddress{#1}}}
\newcommand{\headers}[2]{}
\newenvironment{keywords}{\par\small\noindent\textbf{Key words.}\ }{\par}
\newenvironment{MSCcodes}{\par\small\noindent\textbf{MSC codes.}\ }{\par}
\date{}

\usepackage{amsfonts}
\usepackage{amssymb}
\usepackage{amsbsy}
\usepackage{graphicx}
\usepackage{epstopdf}
\usepackage{booktabs}
\usepackage{algorithmic}

\ifpdf
  \DeclareGraphicsExtensions{.eps,.pdf,.png,.jpg}
\else
  \DeclareGraphicsExtensions{.eps}
\fi

\DeclareMathOperator{\tr}{tr}
\DeclareMathOperator{\diag}{diag}
\DeclareMathOperator{\blkdiag}{blkdiag}
\DeclareMathOperator{\Cov}{Cov}
\DeclareMathOperator{\Var}{Var}
\newcommand{\E}{\mathbb{E}}

\newcommand{\muMF}{\hat{\mu}_{\mathrm{MF}}}
\newcommand{\VM}{V_M}
\newcommand{\gamS}[1]{\gamma^{\mathcal{S}}_{1,#1}}
\newcommand{\gfull}[1]{\gamma^{\mathrm{full}}_{1,#1}}
\newcommand{\gblock}[1]{\gamma^{\mathrm{block}}_{1,#1}}
\newcommand{\gdiag}[1]{\gamma^{\mathrm{diag}}_{1,#1}}
\newcommand{\gscalar}[1]{\gamma^{\mathrm{scalar}}_{1,#1}}

\newsiamremark{remark}{Remark}
\newsiamthm{assumption}{Assumption}
\crefname{assumption}{Assumption}{Assumptions}

\headers{MFMC for High Dimensional QoI}{Z. Dai, H. A. Tchelepi, and D. M. Tartakovsky}

\title{Multifidelity Monte Carlo Estimator for High-Dimensional Quantities of Interest}

\author{Zhibo Dai\thanks{Department of Energy Science and Engineering, Stanford
  University, Stanford, CA 94305 USA (\email{zhibodai@stanford.edu},
  \email{tchelepi@stanford.edu}),
  \email{tartakovsky@stanford.edu}.}
    \and Hamdi A. Tchelepi\footnotemark[1]
  \and Daniel M. Tartakovsky\footnotemark[1]
}

\ifpdf
\hypersetup{
  pdftitle={Multifidelity Monte Carlo Estimator for High-Dimensional Quantities of Interest},
  pdfauthor={Z. Dai, H. A. Tchelepi, and D. M. Tartakovsky}
}
\fi

\begin{document}

\maketitle

\begin{abstract}
Multifidelity Monte Carlo (MFMC) reduces the cost of estimating the statistics of quantities of interest (QoIs) computed via expensive high-fidelity model by using cheaper low-fidelity models as control variates. We extend MFMC to vector-valued QoIs by replacing the scalar control-variate coefficients with matrices and minimizing a weighted mean-squared error. The optimal full-matrix coefficients are available in closed form, and their variance reduction is summarized by an effective correlation that takes the place of the squared Pearson correlation in the closed-form MFMC design. Restricting the coefficients to block-diagonal, diagonal, and scalar classes gives a nested hierarchy of estimators with increasing variance reduction and increasing pilot-data demand, with closed-form coefficients for compatible weightings. For a scalar QoI obtained by linearly projecting the state, correcting the full state before projecting is never worse at the population level. Experiments on a Gaussian problem, a reduced-order model of an advection--diffusion problem, and a coarse-grid two-phase flow problem show that the best class depends on the pilot size relative to the number of covariance quantities it must estimate, and that reusing the pilot samples lowers the mean-squared error.
\end{abstract}



\section{Introduction}
\label{sec:introduction}

Monte Carlo estimation remains a central tool for forward uncertainty
quantification because it is non-intrusive and its convergence rate is
insensitive to the dimension of the random input. Its usefulness is limited when each sample requires an expensive high-fidelity
simulation. Such difficulty arises in computational fluid
dynamics, structural mechanics, subsurface flow, plasma physics, and
cardiovascular modeling, where a single run usually requires solving a PDE and returns a solution field or a collection of coupled quantities of interest
\cite{konrad2022data,fleeter2020multilevel,gruber2026advancements}. In such
problems the quantity of interest is naturally vector-valued. Let
$f^{(1)}(Z)\in\mathbb{R}^{d}$ denote the high-fidelity model output for a
random input $Z$; the goal of this work is to estimate
\begin{equation*}
  \mu \;=\; \mathbb{E}\!\left[f^{(1)}(Z)\right]
\end{equation*}
under a prescribed computational budget.

Multilevel and multifidelity methods reduce the cost of this estimation
problem by replacing many high-fidelity evaluations with evaluations of
cheaper, correlated models. Multilevel Monte Carlo exploits a hierarchy of
discretization levels through a telescoping decomposition
\cite{giles2008multilevel,giles2015multilevel}. Multifidelity methods are
broader: the auxiliary models may be reduced-order models,
simplified-physics models, data-fit surrogates, or machine-learning models
\cite{peherstorfer2018survey,jabarullah2019machine,zanoni2024improved}.
Within this family, multifidelity Monte Carlo (MFMC) uses low-fidelity
outputs as control variates while retaining unbiasedness with respect to
the high-fidelity mean. The two-model construction 
\cite{ng2014multifidelity} and the model-management framework \cite{peherstorfer2016optimal} show
that, for scalar quantities of interest, one obtains closed-form
control-variate weights and closed-form sample-allocation rules
determined entirely by model costs and correlation coefficients, together
with an explicit model-selection criterion. Extensions include convergence
analysis \cite{peherstorfer2018convergence}, adaptive low-fidelity model
construction \cite{peherstorfer2019multifidelity,farcas2023context},
variance and sensitivity-index estimation \cite{qian2018multifidelity}, and
combined multilevel--multifidelity estimators
\cite{geraci2015multifidelity,geraci2017multifidelity}. This analytic
structure, with optimal weights, allocation, and model selection available in
closed form, has made scalar MFMC widely used in large-scale
applications.

The scalar formulation, however, does not directly resolve the multi-output
case, a limitation identified explicitly in \cite{peherstorfer2016optimal}.
Running scalar MFMC independently for each output component may select
different model orderings and different sample allocations for different
components, even though a single vector-valued model evaluation produces
all components simultaneously. Moreover, correlations among output
components carry variance-reduction information that componentwise
treatments cannot exploit: a low-fidelity component that is only weakly
correlated with its own high-fidelity counterpart may still be strongly
predictive of other high-fidelity components.

Existing approaches to vector-valued multifidelity estimation can be
organized by two questions: whether the estimator exploits correlations
\emph{across} output components, and whether its design (control-variate
coefficients, sample allocation, and model selection) is available in
closed form. A first group of methods retains a tractable design but treats
the output components separately. It includes an 
extension of the nested MFMC estimator to
spatially dependent quantities of interest by attaching an independent
scalar coefficient to each output component \cite{quaglino2019high}. Aggregating the component-wise
errors into a spatially integrated mean-squared error yields a common sample
allocation in closed form, driven by a variance-weighted average of the
component-wise squared correlations. In the terminology of \cite{quaglino2019high}, the
estimator is the \emph{diagonal} control variate under a diagonal weighting
matrix: cross-component covariance never enters the design. Multilevel best
linear unbiased estimators (MLBLUE) \cite{schaden2020multilevel} treat
multifidelity estimation as generalized least-squares estimation over groups
of jointly sampled models. Multi-output MLBLUE \cite{croci2023multi}, i.e., MLBLUE  for multiple quantities of interest, construct one estimator per output and couple the estimators only through a shared sample allocation, which is computed by
semidefinite programming. The resulting
estimators are best linear unbiased individually but not jointly, because
correlations between different outputs are not exploited in order to limit
the computational cost.

A second group exploits cross-output correlations through control-variate
matrix coefficients, in the spirit of classical multivariate control variates
\cite{rubinstein1985efficiency}, but gives up the closed-form design. The
approximate control variate (ACV) framework \cite{gorodetsky2020generalized}
contains MFMC as a special case and relaxes its nested sampling structure,
which otherwise limits the achievable variance reduction to that of a single
low-fidelity model with known mean. Extension of ACV to multi-output, termed multi-statistic estimation or MOACV in \cite{dixon2024covariance}, result in
control-variate coefficients that are full matrices. These are given by the multivariate
control-variate formula, which minimizes both the determinant and the trace of
the estimator covariance; restricted to the MFMC sampling scheme at a fixed allocation,
MOACV therefore reduces to the full-matrix estimator studied here. The
sample allocation, however, is obtained by numerically minimizing the
determinant of the estimator covariance, and allocation under a trace
criterion is left open.  Multivariate MLBLUEs with a
hierarchy of control-variate coefficients of increasing complexity (scalar,
element-wise, spectral, and general matrix) under a trace criterion \cite{destouches2023multivariate} do not resolve the allocation problem for
general matrix control-variate coefficients, and these coefficients are intractable in large
dimension. Within nested MFMC itself, the cross-covariance vector of a
multifidelity linear regression can be estimated with control-variate matrix coefficients
$A_k=\Gamma_{1k}\Gamma_{kk}^{-1}$, which minimize the estimator covariance
in the Loewner order, and with a trace-optimal scalar control-variate
coefficient as a cheaper alternative \cite{qian2024multifidelity}. Their sample allocation, however, is taken from scalar
MFMC and computed from the correlations of the scalar model outputs rather
than of the vector being estimated.

Thus, the closed-form design of scalar MFMC has so far been extended to
vector-valued outputs only for componentwise coefficients
\cite{quaglino2019high}, while the methods that exploit
cross-component correlations rely on numerical or heuristic allocation. To
our knowledge, no existing method combines cross-component control variates
with the closed-form design of scalar MFMC. The ACV and MLBLUE frameworks contain the nested
estimator in their search space and are therefore never worse when the
exact model covariances are known, but their additional flexibility must be paid
for with covariance information that is itself estimated. 
For scalar quantities of interest, these more general estimators can have much larger variance than
MFMC \cite{stanley2026sample}, which the authors attribute to allocations that overfit spurious pilot
cross-correlations, whereas the sample-allocation constraints of MFMC act as a regularization. A closed-form
nested design is therefore well suited to the regime where pilot data
are scarce.

The same consideration governs how much cross-component information should
be used. A full matrix control variate reduces the population error at least
as much as any scalar or componentwise correction, but its covariance
matrices must be estimated from pilot data, and classical theory shows that
estimating the control matrix degrades the theoretical
gain when the number of parameters is large relative to the sample size
\cite{lavenberg1982statistical,rubinstein1985efficiency,venkatraman1986efficiency}.
The multifidelity literature reports the same phenomenon: the optimal matrix
coefficient of \cite{qian2024multifidelity} becomes unstable once its
covariances are estimated, already for $d=21$ outputs and $100$ pilot
samples; multi-output ACV estimators can perform worse than plain Monte
Carlo when pilot samples are few \cite{dixon2024covariance}; and estimated
ACV weights require a minimum sample size to guarantee variance reduction
\cite{pham2022ensemble}. These studies involve at most a few tens of
outputs. For field-valued outputs with really large $d$ and
$n_{\mathrm{pilot}}\ll d$, the estimated $d\times d$ covariance is
rank-deficient and the full correction and allocation can be
unstable, whereas scalar and componentwise estimators remain stable but
discard cross-component information. A practical vector-valued estimator
must therefore balance \emph{population optimality} against
\emph{pilot-estimation stability}, and should allow intermediate choices
between the two extremes.

In this paper we develop a vector-valued MFMC method that retains the
nested sampling structure and the closed-form design of scalar MFMC, while
making this tradeoff explicit and controllable through the structure of the
control-variate coefficients. We replace the scalar coefficients by
matrices $A_i\in\mathbb{R}^{d\times d}$ and minimize the weighted
mean-squared error $\mathbb{E}\|\hat{\mu}-\mu\|_M^2$, where $M\succeq 0$
encodes the relative importance of the output components. The variance
reduction achievable with model $i$ is summarized by an \emph{effective
correlation}, which for the full matrix class is given by
\begin{equation*}
  \gamma_{1,i} \;=\;
  \frac{\operatorname{tr}\!\bigl(M\,\Sigma_{1,i}\Sigma_i^{-1}\Sigma_{i,1}\bigr)}
       {\operatorname{tr}\!\bigl(M\,\Sigma_1\bigr)}\;\in[0,1],
\end{equation*}
where $\Sigma_i$ is the covariance of the $i$th model output and
$\Sigma_{1,i}$ its cross-covariance with the high-fidelity output. The
effective correlation replaces the squared Pearson correlation
$\rho_{1,i}^2$ in the closed-form allocation rule, the cost--correlation
ordering condition, and the model-selection criterion of
\cite{peherstorfer2016optimal}. Restricting the coefficients to scalar,
diagonal, or block-diagonal matrices yields a hierarchy of estimators that
trades cross-component information for pilot-data demand, with closed-form
coefficients and effective correlations whenever $M$ is compatible with the
imposed structure. The
diagonal member recovers the estimator of \cite{quaglino2019high},
and the scalar theory is recovered for $d=1$.

The contributions of this paper are as follows.
\begin{enumerate}
  \item A nested vector-valued MFMC estimator with matrix control-variate
        coefficients is formulated under a weighted mean-squared-error
        criterion. The optimal matrix control variate coefficients are obtained in
        closed form, and the resulting effective correlation takes the place
        of the squared Pearson correlation, so that the closed-form sample
        allocation, the cost--correlation ordering condition, and the model
        selection of \cite{peherstorfer2016optimal} carry over unchanged.
  \item Restricting the coefficients to block-diagonal, diagonal, and scalar
        matrices gives a nested hierarchy of estimators that trades variance
        reduction for a smaller pilot-data demand. For a fixed sample
        allocation, the MSE increase of each reduced class over the full class
        is given exactly in terms of the gaps between their effective
        correlations.
  \item For a scalar quantity of interest obtained by linearly projecting
        the full high-dimensional state, correcting the full state and then
        projecting is shown to be, with population-optimal coefficients, at
        least as accurate as scalar MFMC on the projected outputs.
  \item A pilot-based algorithm is developed, with optional reuse of the
        pilot samples, and with an exact reduction of the full class to a
        regression on the reduced coordinates when the low-fidelity output
        lies in a known low-dimensional subspace. Experiments on a Gaussian
        problem, a POD--Galerkin ROM of an advection--diffusion problem,
        and a coarse-grid model of two-phase gas injection show that the
        coefficient class should be matched to the pilot size relative to
        the number of covariance quantities it must estimate.
\end{enumerate}

The remainder of the paper is organized as follows.
Section~\ref{sec:scalar-mfmc} reviews the scalar MFMC estimator and its
optimal allocation rule. Section~\ref{sec:vector-mfmc} develops the
vector-valued estimator, derives the optimal matrix coefficients and
effective correlations, analyzes the structured coefficient classes and
their hierarchy, treats linearly mapped quantities of interest, and describes
a pilot-based implementation. Section~\ref{sec:numerics} presents the
numerical experiments, and Section~\ref{sec:conclusion} is the conclusion.

\section{Scalar-valued MFMC}
\label{sec:scalar-mfmc}
Before developing the vector-valued generalization, we review the scalar-valued
MFMC method 
\cite{peherstorfer2016optimal}. The purpose of this section is to fix
the sampling structure and notation that the vector-valued estimator inherits,
and to identify the two quantities of the scalar theory, the optimal
control-variate coefficient and the squared correlation coefficient, for which
the vector-valued theory must develop counterparts.

\subsection{Scalar-valued MFMC estimator}
We consider a collection of $k$ scalar-valued models
\[
f^{(i)} : \mathcal{D} \to \mathbb{R}, \quad i = 1, \dots, k,
\]
where $f^{(1)}$ is a high-fidelity (HF) model and $f^{(2)}, \dots, f^{(k)}$ are low-fidelity (LF) surrogate models with reduced computational cost. Each evaluation of model $f^{(i)}$ incurs a cost $w_i > 0$, and the models are ordered so that
\[
w_1 \gg w_2 \geq \cdots \geq w_k.
\]

Given a random input $Z \in \mathcal{D}$, the goal is to estimate the expectation of the HF model output
\[
s = \E[f^{(1)}(Z)]
\]
under a fixed computational budget $p$. We assume that each LF model is correlated with the HF model,
\[
\rho_{1,i}^2 > 0, \quad i = 2, \dots, k,
\]
where $\rho_{1,i}$ denotes the Pearson correlation coefficient between $f^{(1)}(Z)$ and $f^{(i)}(Z)$, so that the LF outputs carry information about the HF output and can serve as control variates.

The MFMC estimator uses a \emph{nested} set of samples: model $f^{(i)}$ is evaluated at the first $m_i$ of a common pool of i.i.d.\ samples $z_1, \dots, z_{m_k}$, where
\[
0 < m_1 \leq m_2 \leq \cdots \leq m_k,
\]
so that the cheap models are evaluated more often than the expensive ones. The corresponding Monte Carlo sample means are
\[
\bar y^{(i)}_{m_i} = \frac{1}{m_i} \sum_{j=1}^{m_i} f^{(i)}(z_j),
\qquad i = 1,\dots,k.
\]

The MFMC estimator is
\begin{equation}
\label{eq:scalar-mfmc}
\hat{s} = \bar y^{(1)}_{m_1} + \sum_{i=2}^k \alpha_i \left( \bar y^{(i)}_{m_i} - \bar y^{(i)}_{m_{i-1}} \right),
\end{equation}
where $\alpha_i \in \mathbb{R}$ are control-variate coefficients. The leading term is the plain Monte Carlo estimator built from the $m_1$ HF samples; each correction term uses the difference between two sample means of the \emph{same} LF model, so it has mean zero but is correlated with the leading term and can cancel part of its statistical error. The design questions are how to choose the coefficients $\alpha_i$ and the sample sizes $m_i$. We summarize the answers from \cite{peherstorfer2016optimal}; proofs can be found there.

\begin{lemma}[Unbiasedness]
For any coefficients $\alpha_i$ and any nested sample sizes, the MFMC estimator is unbiased:
\[
\E[\hat{s}] = \E[f^{(1)}(Z)] = s.
\]
\end{lemma}

This follows from linearity of expectation: the mean of each correction term is $\alpha_i(\E[f^{(i)}]-\E[f^{(i)}])=0$. Consequently, the mean-squared error of $\hat{s}$ equals its variance.

\begin{lemma}[Variance expression]
The variance of the MFMC estimator is
\[
\Var[\hat{s}] = \frac{\sigma_1^2}{m_1}
+ \sum_{i=2}^k \left( \frac{1}{m_{i-1}} - \frac{1}{m_i} \right)
\left( \alpha_i^2 \sigma_i^2 - 2 \alpha_i \rho_{1,i} \sigma_1 \sigma_i \right),
\]
where $\sigma_i^2 = \Var[f^{(i)}(Z)]$.
\end{lemma}

The first term is the variance of the HF Monte Carlo estimator; each summand quantifies how much variance the $i$th correction term adds or removes, depending on the coefficient $\alpha_i$ and on the sample-size gap between levels $i-1$ and $i$. The optimal estimator minimizes this variance subject to the cost constraint:
\[
\min_{\{m_i\}, \{\alpha_i\}} \ \Var[\hat{s}]
\quad \text{s.t.} \quad \sum_{i=1}^k w_i m_i = p,
\qquad 0 < m_1 \le \cdots \le m_k.
\]

\begin{theorem}[Optimal scalar MFMC estimator \cite{peherstorfer2016optimal}]
\label{thm:scalar-mfmc}
Assume the models are ordered such that
\[
1 = \rho_{1,1}^2 > \rho_{1,2}^2 > \cdots > \rho_{1,k}^2 > \rho^2_{1,k+1} := 0,
\]
and that the costs satisfy the cost--correlation condition
\[
\frac{w_{i-1}}{w_i} >
\frac{\rho_{1,i-1}^2 - \rho_{1,i}^2}{\rho_{1,i}^2 - \rho_{1,i+1}^2},
\quad i = 2, \dots, k.
\]
Then the optimal control-variate coefficients are
\[
\alpha_i^\star = \frac{\rho_{1,i} \sigma_1}{\sigma_i}, \quad i = 2, \dots, k,
\]
and the optimal sample sizes are
\[
m_i^\star = m_1^\star \, r_i, \quad i = 1, \dots, k,
\qquad
r_i = \sqrt{
\frac{w_1 (\rho_{1,i}^2 - \rho_{1,i+1}^2)}
{w_i (1 - \rho_{1,2}^2)}
},
\]
where the number of HF samples is determined by the budget constraint:
\[
m_1^\star = \frac{p}{\sum_{i=1}^k w_i r_i}.
\]
\end{theorem}

The optimal coefficient $\alpha_i^\star$ is available in closed form, and the
optimal allocation depends on the models only through the costs $w_i$ and the
squared correlations $\rho_{1,i}^2\in[0,1]$. For vector-valued outputs,
$\alpha_i$ becomes a matrix, and the counterpart of $\rho_{1,i}^2$ is not
obvious. Section~\ref{sec:vector-mfmc} shows that the optimal matrix is a
multivariate regression coefficient and that the role of $\rho_{1,i}^2$ is
played by an effective correlation $\gamma_{1,i}\in[0,1]$.

\section{Vector-valued MFMC}
\label{sec:vector-mfmc}

\subsection{Problem setup and the estimator}
\label{sec:vv-setup}

We now consider a collection of $k$ models with vector-valued output,
\[
f^{(i)} : \mathcal{D} \to \mathbb{R}^d, \quad i = 1, \dots, k, \quad d>1,
\]
where, as before, $f^{(1)}$ is the high-fidelity (HF) model, $f^{(2)}, \dots, f^{(k)}$ are lower-fidelity (LF) surrogate models, and the evaluation costs $w_i > 0$ satisfy
\[
w_1 \gg w_2 \geq \cdots \geq w_k.
\]
Given a random input $Z \in \mathcal{D}$, the goal is to estimate the mean of the HF output,
\[
\mu = \E[f^{(1)}(Z)] \in \mathbb{R}^d,
\]
under a fixed computational budget $p$.

The sampling structure is identical to the scalar case: model $f^{(i)}$ is evaluated at the first $m_i$ of a common pool of i.i.d.\ samples $z_1, \dots, z_{m_k}$, where
\[
0 < m_1 \leq m_2 \leq \cdots \leq m_k,
\]
and the sample means are
\[
\bar y^{(i)}_{m_i} = \frac{1}{m_i} \sum_{j=1}^{m_i} f^{(i)}(z_j) \in \mathbb{R}^d.
\]
The only structural change relative to the scalar estimator \eqref{eq:scalar-mfmc} is that each scalar coefficient $\alpha_i$ is replaced by a matrix: the \emph{vector-valued MFMC estimator} is
\begin{equation}
\label{eq:vector-mfmc}
\muMF =\bar y^{(1)}_{m_1} + \sum_{i=2}^k A_i \left(\bar y^{(i)}_{m_i} -\bar y^{(i)}_{m_{i-1}} \right),
\end{equation}
where $A_i \in \mathbb{R}^{d \times d}$ are control-variate matrices. A matrix coefficient allows each component of the HF estimate to be corrected by \emph{any} linear combination of the LF components, so that correlations \emph{across} output components, and not only within each component, can be exploited for variance reduction.

\paragraph{Covariance notation}
For each model output, define
\[
\mu_i := \E[f^{(i)}(Z)], \qquad i=1,\dots,k,
\]
so that the target QoI is \(\mu=\mu_1\). For any two fidelity levels \(l,t\in\{1,\dots,k\}\), define the cross-covariance matrix
\[
\Sigma_{l,t}
:=
\Cov\!\left(f^{(l)}(Z),f^{(t)}(Z)\right)
=
\E\!\left[
\bigl(f^{(l)}(Z)-\mu_l\bigr)
\bigl(f^{(t)}(Z)-\mu_t\bigr)^\top
\right]
\in \mathbb{R}^{d\times d},
\]
and abbreviate the covariance matrix of the \(i\)th model output by
\[
\Sigma_i := \Sigma_{i,i}
=
\Cov\!\left(f^{(i)}(Z)\right).
\]
With this convention,
\(
\Sigma_{t,l} = \Sigma_{l,t}^{\top}.
\)

The relative importance of the output components is encoded by a user-chosen symmetric weighting matrix $M$, through which all errors will be measured (Definition~\ref{def:weighted-mse} below). Typical choices are $M=I$ (all components equally important), a diagonal $M$ (componentwise weights), or $M = B^\top B$ and $M = cc^\top$ when the actual QoI is a linear function of the model output (Section~\ref{sec:special-cases}). We make the following assumptions:

\begin{assumption}
\label{ass:standing}
Throughout Section~\ref{sec:vector-mfmc} we assume:
(i) all model outputs have finite second moments;
(ii) $\Sigma_i$ is symmetric positive definite ($\Sigma_i \succ 0$) for every $i$ whose inverse $\Sigma_i^{-1}$ appears;
(iii) the weighting matrix satisfies $M \succeq 0$ (symmetric positive semidefinite) and
\[
\VM := \tr(M\Sigma_1) > 0,
\]
i.e.\ the HF output has nonzero weighted variance, so that the estimation problem is nondegenerate.
\end{assumption}

The quantity $\VM$ is the natural vector-valued generalization of the HF variance $\sigma_1^2$: it aggregates the variances of all components of $f^{(1)}(Z)$, weighted by $M$. It will serve as the reference against which all variance reductions are measured.

\subsection{Weighted mean-squared error of the estimator}
\label{sec:vv-mse}

We first show that the matrix control variates do not affect unbiasedness, and then derive the exact covariance of the estimator. The derivation has the same structure as in the scalar case.

\begin{lemma}[Unbiasedness]
For any matrices $A_i$ and any nested sample sizes, the estimator \eqref{eq:vector-mfmc} is unbiased:
\[
\E[\muMF] = \mu.
\]
\end{lemma}

\begin{proof}
By linearity of expectation, $\E[\bar y^{(i)}_{m}] = \mu_i$ for any $m$, so each correction term satisfies $\E\big[A_i (\bar y^{(i)}_{m_i} - \bar y^{(i)}_{m_{i-1}})\big] = A_i(\mu_i - \mu_i) = 0$, and $\E[\muMF]=\mu_1=\mu$.
\end{proof}

For a vector-valued estimator, the mean-squared error depends on the choice of norm; we measure errors in the norm induced by $M$.

\begin{definition}[Weighted mean-squared error]
\label{def:weighted-mse}
Let $M \succeq 0$. The weighted mean-squared error (MSE) of an estimator $\hat{\mu}$ of $\mu$ is
\[
\mathcal{E}_M(\hat{\mu}) :=\E\big[\|\hat{\mu}-\mu\|_M^2\big]
= \E\big[(\hat{\mu}-\mu)^\top M(\hat{\mu}-\mu)\big].
\]
\end{definition}

\begin{lemma}
\label{lem:mse-trace}
If $\hat{\mu}$ is unbiased, then
\[
\mathcal{E}_M(\hat{\mu}) = \tr\!\big(M \Cov(\hat{\mu})\big).
\]
\end{lemma}
\begin{proof}
Writing the squared $M$-norm as a trace,
\[
\mathcal{E}_M(\hat{\mu})
=
\E\big[\tr\!\big(M(\hat{\mu}-\mu)(\hat{\mu}-\mu)^\top\big)\big]
=
\tr\!\left(
M\,\E\big[(\hat{\mu}-\mu)(\hat{\mu}-\mu)^\top\big]
\right).
\]
Moreover,
\[
\E\big[(\hat{\mu}-\mu)(\hat{\mu}-\mu)^\top\big]
=
\Cov(\hat{\mu})
+
(\E[\hat{\mu}]-\mu)(\E[\hat{\mu}]-\mu)^\top .
\]
If \(\hat{\mu}\) is unbiased, then \(\E[\hat{\mu}]=\mu\), so the bias term vanishes and
\(
\mathcal{E}_M(\hat{\mu})
=
\tr\!\big(M\Cov(\hat{\mu})\big).
\)
\end{proof}
The next lemma is the vector-valued analogue of \cite[Lemma~3.2]{peherstorfer2016optimal}.

\begin{lemma}[Covariance of two sample means on shared samples]
\label{lem:sample-mean-cov}
Let \(\bar y_{m_i}^{(l)}\) and \(\bar y_{m_j}^{(t)}\) be sample means computed from the first $m_i$ and the first $m_j$ elements, respectively, of the common i.i.d.\ samples $z_1,\dots,z_{m_k}$. Then
\[
\Cov\!\left[\bar y_{m_i}^{(l)}, \bar y_{m_j}^{(t)}\right]
=
\frac{1}{\max\{m_i,m_j\}}\, \Sigma_{l,t}.
\]
\end{lemma}

\begin{proof}
By bilinearity of the covariance,
\[
\Cov\!\left[\bar y_{m_i}^{(l)}, \bar y_{m_j}^{(t)}\right]
=
\frac{1}{m_i m_j}
\sum_{a=1}^{m_i}\sum_{b=1}^{m_j}
\Cov\!\left[f^{(l)}(z_a), f^{(t)}(z_b)\right].
\]
Since the samples are i.i.d., the terms with \(a\neq b\) vanish, and each term with $a=b$ equals $\Sigma_{l,t}$. The two sample means share the first \(\min\{m_i,m_j\}\) samples, so exactly \(\min\{m_i,m_j\}\) terms survive:
\[
\Cov\!\left[\bar y_{m_i}^{(l)}, \bar y_{m_j}^{(t)}\right]
=
\frac{\min\{m_i,m_j\}}{m_i m_j} \Sigma_{l,t}
=
\frac{1}{\max\{m_i,m_j\}} \Sigma_{l,t}.
\]
\end{proof}

With Lemma~\ref{lem:sample-mean-cov} in hand, the covariance of the MFMC estimator follows by direct expansion. The key structural facts, established in the proof, are that the correction terms $A_i(\bar y_{m_i}^{(i)}-\bar y_{m_{i-1}}^{(i)})$ are mutually uncorrelated across levels, and that the cross-covariance of each with the HF sample mean is $-(1/m_{i-1}-1/m_i)\Sigma_{1,i}A_i^\top$, so that a suitable $A_i$ cancels part of the HF sampling error.

\begin{lemma}[Covariance of the vector-valued MFMC estimator]
\label{lem:cov_vector_mfmc}
The covariance matrix of the estimator $\muMF$ in \eqref{eq:vector-mfmc} is
\[
\Cov\!\left[\muMF\right]
=
\frac{1}{m_1}\Sigma_1
+
\sum_{i=2}^k
\left(\frac{1}{m_{i-1}}-\frac{1}{m_i}\right)
\Bigl(
A_i \Sigma_i A_i^\top
-
A_i \Sigma_{i,1}
-
\Sigma_{1,i} A_i^\top
\Bigr).
\]
\end{lemma}

\begin{proof}
Let
\[
\Delta_i
:=
\bar y_{m_i}^{(i)}-\bar y_{m_{i-1}}^{(i)},
\qquad i=2,\dots,k,
\]
so that
\[
\muMF
=
\bar y_{m_1}^{(1)}
+
\sum_{i=2}^k A_i\Delta_i .
\]
Expanding the covariance gives
\[
\begin{aligned}
\Cov[\muMF]
&=
\operatorname{Cov}\!\left[\bar y_{m_1}^{(1)}\right]
+
\sum_{i=2}^k
\operatorname{Cov}[A_i\Delta_i]
\\
&\quad
+
\sum_{i=2}^k
\operatorname{Cov}\!\left[\bar y_{m_1}^{(1)},A_i\Delta_i\right]
+
\sum_{i=2}^k
\operatorname{Cov}\!\left[A_i\Delta_i,\bar y_{m_1}^{(1)}\right]
\\
&\quad
+
\sum_{\substack{i,j=2\\ i\ne j}}^k
\operatorname{Cov}[A_i\Delta_i,A_j\Delta_j].
\end{aligned}
\]

We now compute each term. By Lemma~\ref{lem:sample-mean-cov},
\[
\operatorname{Cov}\!\left[\bar y_{m_1}^{(1)}\right]
=
\frac{1}{m_1}\Sigma_1.
\]

Next, for each \(i=2,\dots,k\),
\[
\begin{aligned}
\operatorname{Cov}[\Delta_i]
&=
\operatorname{Cov}\!\left[
\bar y_{m_i}^{(i)}-\bar y_{m_{i-1}}^{(i)}
\right]
\\
&=
\operatorname{Cov}\!\left[\bar y_{m_i}^{(i)}\right]
+
\operatorname{Cov}\!\left[\bar y_{m_{i-1}}^{(i)}\right]
-
\operatorname{Cov}\!\left[\bar y_{m_i}^{(i)},\bar y_{m_{i-1}}^{(i)}\right]
-
\operatorname{Cov}\!\left[\bar y_{m_{i-1}}^{(i)},\bar y_{m_i}^{(i)}\right].
\end{aligned}
\]
Using Lemma~\ref{lem:sample-mean-cov} and \(m_{i-1}\le m_i\), we have
\[
\operatorname{Cov}\!\left[\bar y_{m_i}^{(i)}\right]
=
\frac{1}{m_i}\Sigma_i,
\qquad
\operatorname{Cov}\!\left[\bar y_{m_{i-1}}^{(i)}\right]
=
\frac{1}{m_{i-1}}\Sigma_i,
\]
and
\[
\operatorname{Cov}\!\left[\bar y_{m_i}^{(i)},\bar y_{m_{i-1}}^{(i)}\right]
=
\operatorname{Cov}\!\left[\bar y_{m_{i-1}}^{(i)},\bar y_{m_i}^{(i)}\right]
=
\frac{1}{m_i}\Sigma_i.
\]
Therefore,
\[
\operatorname{Cov}[\Delta_i]
=
\left(
\frac{1}{m_{i-1}}-\frac{1}{m_i}
\right)\Sigma_i.
\]
Hence
\[
\operatorname{Cov}[A_i\Delta_i]
=
A_i\operatorname{Cov}[\Delta_i]A_i^\top
=
\left(
\frac{1}{m_{i-1}}-\frac{1}{m_i}
\right)
A_i\Sigma_iA_i^\top.
\]

We also compute the covariance between the high-fidelity sample mean and the increment \(\Delta_i\). Since \(m_1\le m_{i-1}\le m_i\), Lemma~\ref{lem:sample-mean-cov} gives
\[
\operatorname{Cov}\!\left[
\bar y_{m_1}^{(1)},\bar y_{m_i}^{(i)}
\right]
=
\frac{1}{m_i}\Sigma_{1,i},
\]
and
\[
\operatorname{Cov}\!\left[
\bar y_{m_1}^{(1)},\bar y_{m_{i-1}}^{(i)}
\right]
=
\frac{1}{m_{i-1}}\Sigma_{1,i}.
\]
Thus
\[
\begin{aligned}
\operatorname{Cov}\!\left[\bar y_{m_1}^{(1)},\Delta_i\right]
&=
\operatorname{Cov}\!\left[
\bar y_{m_1}^{(1)},\bar y_{m_i}^{(i)}
\right]
-
\operatorname{Cov}\!\left[
\bar y_{m_1}^{(1)},\bar y_{m_{i-1}}^{(i)}
\right]
\\
&=
-\left(
\frac{1}{m_{i-1}}-\frac{1}{m_i}
\right)\Sigma_{1,i}.
\end{aligned}
\]
Similarly,
\[
\operatorname{Cov}\!\left[\Delta_i,\bar y_{m_1}^{(1)}\right]
=
-\left(
\frac{1}{m_{i-1}}-\frac{1}{m_i}
\right)\Sigma_{i,1}.
\]
Therefore,
\[
\operatorname{Cov}\!\left[\bar y_{m_1}^{(1)},A_i\Delta_i\right]
=
\operatorname{Cov}\!\left[\bar y_{m_1}^{(1)},\Delta_i\right]A_i^\top
=
-\left(
\frac{1}{m_{i-1}}-\frac{1}{m_i}
\right)\Sigma_{1,i}A_i^\top,
\]
and
\[
\operatorname{Cov}\!\left[A_i\Delta_i,\bar y_{m_1}^{(1)}\right]
=
A_i\operatorname{Cov}\!\left[\Delta_i,\bar y_{m_1}^{(1)}\right]
=
-\left(
\frac{1}{m_{i-1}}-\frac{1}{m_i}
\right)A_i\Sigma_{i,1}.
\]

It remains to show that the increments are mutually uncorrelated. Let \(i<j\). Since the sample sizes are nested,
\[
m_{i-1}\le m_i\le m_{j-1}\le m_j.
\]
Using Lemma~\ref{lem:sample-mean-cov}, we obtain
\[
\begin{aligned}
\operatorname{Cov}[\Delta_i,\Delta_j]
&=
\operatorname{Cov}\!\left[
\bar y_{m_i}^{(i)}-\bar y_{m_{i-1}}^{(i)},
\bar y_{m_j}^{(j)}-\bar y_{m_{j-1}}^{(j)}
\right]
\\
&=
\operatorname{Cov}\!\left[\bar y_{m_i}^{(i)},\bar y_{m_j}^{(j)}\right]
-
\operatorname{Cov}\!\left[\bar y_{m_i}^{(i)},\bar y_{m_{j-1}}^{(j)}\right]
\\
&\quad
-
\operatorname{Cov}\!\left[\bar y_{m_{i-1}}^{(i)},\bar y_{m_j}^{(j)}\right]
+
\operatorname{Cov}\!\left[\bar y_{m_{i-1}}^{(i)},\bar y_{m_{j-1}}^{(j)}\right]
\\
&=
\frac{1}{m_j}\Sigma_{i,j}
-
\frac{1}{m_{j-1}}\Sigma_{i,j}
-
\frac{1}{m_j}\Sigma_{i,j}
+
\frac{1}{m_{j-1}}\Sigma_{i,j}
\\
&=0.
\end{aligned}
\]
Therefore,
\[
\operatorname{Cov}[A_i\Delta_i,A_j\Delta_j]=0,
\qquad i\ne j.
\]

Substituting all these identities into the covariance expansion yields
\[
\operatorname{Cov}[\muMF]
=
\frac{1}{m_1}\Sigma_1
+
\sum_{i=2}^k
\left(
\frac{1}{m_{i-1}}-\frac{1}{m_i}
\right)
\left(
A_i\Sigma_iA_i^\top
-
A_i\Sigma_{i,1}
-
\Sigma_{1,i}A_i^\top
\right).
\]
This completes the proof.
\end{proof}

Combining Lemma~\ref{lem:mse-trace} with Lemma~\ref{lem:cov_vector_mfmc} gives the weighted MSE used in the rest of the paper.

\begin{corollary}[Weighted MSE] \label{corollary: MSE matrix cv}
For fixed sample sizes $m_1 \le \cdots \le m_k$, the weighted MSE of the MFMC estimator \eqref{eq:vector-mfmc} is
\begin{equation}
\label{eq:mse-separable}
\mathcal{E}_M(\muMF)
=
\frac{\VM}{m_1}
+
\sum_{i=2}^k
\delta_i \,
\Phi_i(A_i),
\end{equation}
where $\VM = \tr(M \Sigma_1)$ is the weighted HF variance,
\[
\delta_i := \frac{1}{m_{i-1}} - \frac{1}{m_i} \ge 0,
\qquad
\Phi_i(A)
:=
\tr(M A \Sigma_i A^\top)
- 2 \tr(M A \Sigma_{i,1}),
\qquad i = 2, \dots, k.
\]
\end{corollary}

The form of \eqref{eq:mse-separable} determines the design of the estimator. The first term, $\VM/m_1$, is the weighted MSE of the plain HF Monte Carlo estimator. Each level $i$ then contributes the product of two factors: a factor $\delta_i$ that depends \emph{only on the sample sizes}, and a factor $\Phi_i(A_i)$ that depends \emph{only on the control-variate matrix} $A_i$. The matrix $A_i$ helps exactly when $\Phi_i(A_i)<0$. The two design questions therefore decouple:
\begin{enumerate}
\item For each level $i$, choose $A_i$ to make $\Phi_i(A_i)$ as negative as possible (Section~\ref{sec:optimal-cv}).
\item Then choose the sample sizes $m_1\le\cdots\le m_k$ to minimize the resulting expression under the budget constraint (Section~\ref{sec:allocation}).
\end{enumerate}

\subsection{Optimal control variates and the effective correlation}
\label{sec:optimal-cv}

Since the levels decouple in \eqref{eq:mse-separable}, it suffices to solve, for each $i=2,\dots,k$, the matrix optimization problem
\begin{equation}
\label{eq:cv-problem}
\min_{A \in \mathcal{S}} \; \Phi_i(A),
\qquad
\Phi_i(A)
=
\tr(M A \Sigma_i A^\top) - 2 \tr(M A \Sigma_{i,1}),
\end{equation}
where $\mathcal{S} \subseteq \mathbb{R}^{d\times d}$ is a prescribed \emph{structural class} of admissible coefficient matrices. The most general choice $\mathcal{S}=\mathbb{R}^{d\times d}$ extracts the most variance reduction, but (as Section~\ref{sec:numerics} shows) its optimal coefficient involves a $d\times d$ covariance matrix that must be estimated and inverted from a limited number of pilot samples; when $d$ is large this is unstable or infeasible. Restricting $\mathcal{S}$ trades population-level optimality for estimation robustness. We consider the nested hierarchy
\begin{itemize}
\item \textbf{full:} $\mathcal{S}_{\mathrm{full}} = \mathbb{R}^{d\times d}$;
\item \textbf{block-diagonal:} $\mathcal{S}_{\mathrm{block}} = \{A\in\mathbb{R}^{d\times d}:\; A = \blkdiag(A^{[1]},\dots,A^{[B]})\}$ for a prescribed partition of the $d$ components into $B$ blocks;
\item \textbf{diagonal:} $\mathcal{S}_{\mathrm{diag}} = \{A\in\mathbb{R}^{d\times d}:\; A \text{ diagonal}\}$;
\item \textbf{scalar:} $\mathcal{S}_{\mathrm{scalar}} = \{\alpha I:\; \alpha \in \mathbb{R}\}$,
\end{itemize}
which satisfies
$\mathcal{S}_{\mathrm{scalar}} \subseteq \mathcal{S}_{\mathrm{diag}} \subseteq \mathcal{S}_{\mathrm{block}} \subseteq \mathcal{S}_{\mathrm{full}}$.
Moving up the hierarchy allows more expressive corrections (more cross-component information is used) at the price of more covariance parameters to estimate.

\subsubsection{The full matrix class}
\label{sec:full-class}

We begin with the full matrix class, whose solution is available in closed form.

\begin{lemma}[Optimal full-matrix control variate]
\label{lem:optimal-full}
Let $\Sigma_i \succ 0$ and $M \succeq 0$. Then
\[
A_i^\star=\Sigma_{1,i}\Sigma_i^{-1}
\]
minimizes $\Phi_i$ over $\mathcal{S}_{\mathrm{full}} = \mathbb{R}^{d\times d}$, and the attained minimum is
\begin{equation}
\label{eq:full-min}
\Phi_i(A_i^\star)
=
- \tr\!\big(M\,\Sigma_{1,i}\Sigma_i^{-1}\Sigma_{i,1}\big).
\end{equation}
If in addition $M \succ 0$, the minimizer is unique.
\end{lemma}

\begin{proof}
We complete the square. Since $M$ is symmetric and
$A_i^\star\Sigma_iA_i^{\star\top}=\Sigma_{1,i}\Sigma_i^{-1}\Sigma_{i,1}$,
expanding the right-hand side below and using the cyclicity of the trace
gives, for every $A\in\mathbb{R}^{d\times d}$,
\[
\Phi_i(A)
=
\tr\!\big(M(A-A_i^\star)\,\Sigma_i\,(A-A_i^\star)^\top\big)
-
\tr\!\big(M\Sigma_{1,i}\Sigma_i^{-1}\Sigma_{i,1}\big).
\]
The first term equals $\|M^{1/2}(A-A_i^\star)\Sigma_i^{1/2}\|_F^2\ge 0$ and
vanishes at $A=A_i^\star$, so $A_i^\star$ is a global minimizer with the
minimum value \eqref{eq:full-min}. Moreover, because $\Sigma_i\succ 0$, the
first term vanishes if and only if $M(A-A_i^\star)=0$. If $M\succ 0$, this
forces $A=A_i^\star$, so the minimizer is unique.
\end{proof}

\subsubsection{The effective correlation}
\label{sec:effective-correlation}

Substituting the optimal full-matrix coefficient into the MSE \eqref{eq:mse-separable}, level $i$ contributes
\[
\delta_i\,\Phi_i(A_i^\star)
=
-\,\delta_i\, \tr\!\big(M\,\Sigma_{1,i}\Sigma_i^{-1}\Sigma_{i,1}\big),
\]
i.e.\ the quantity $\tr(M\Sigma_{1,i}\Sigma_i^{-1}\Sigma_{i,1})$ is the amount of weighted variance that model $i$ is able to remove. Normalizing it with the total weighted HF variance $\VM=\tr(M\Sigma_1)$ measures the usefulness of model $i$ on a dimensionless scale from $0$ to $1$, which is the role played by $\rho_{1,i}^2$ in the scalar theory. This motivates the following definition, which we state directly at the level of a general structural class so that the same notion covers the structurally reduced classes as well.

\begin{definition}[Effective correlation of a structural class]
\label{def:effective-correlation}
Let $\mathcal{S}\subseteq\mathbb{R}^{d\times d}$ be a structural class containing $0$. The \emph{effective correlation} between the HF model and the $i$th LF model under the class $\mathcal{S}$ is
\[
\gamS{i}
:=
-\frac{1}{\VM}
\min_{A\in\mathcal{S}} \Phi_i(A),
\qquad i = 2,\dots,k.
\]
Thus $\gamS{i}$ is the largest fraction of the weighted HF variance that the $i$th model can cancel using control variate coefficient matrices from $\mathcal{S}$. For the four classes above we write $\gfull{i}$, $\gblock{i}$, $\gdiag{i}$, and $\gscalar{i}$, respectively.
\end{definition}

By Lemma~\ref{lem:optimal-full} and \eqref{eq:full-min}, the full class admits the closed form
\begin{equation}
\label{eq:gamma-full}
\gfull{i} =
\frac{\tr(M\Sigma_{1,i}\Sigma_i^{-1}\Sigma_{i,1})}
{\tr(M\Sigma_1)} .
\end{equation}
The next lemma collects the basic properties of the effective correlation; in particular it justifies reading $\gamS{i}$ as a generalized squared correlation coefficient.

\begin{lemma}[Properties of the effective correlation]
\label{lem:gamma-properties}
Let $\mathcal{S}'\subseteq\mathcal{S}\subseteq\mathbb{R}^{d\times d}$ be structural classes containing $0$. Then, for each $i=2,\dots,k$,
\[
0 \;\le\; \gamma^{\mathcal{S}'}_{1,i} \;\le\; \gamS{i} \;\le\; \gfull{i} \;\le\; 1.
\]
\end{lemma}

\begin{proof}
\emph{Nonnegativity.} Since $0\in\mathcal{S}'$ and $\Phi_i(0)=0$, we have $\min_{A\in\mathcal{S}'}\Phi_i(A)\le 0$, hence $\gamma^{\mathcal{S}'}_{1,i}\ge 0$.

\emph{Monotonicity.} Minimizing over the larger set $\mathcal{S}$ (respectively $\mathcal{S}_{\mathrm{full}}=\mathbb{R}^{d\times d}$) cannot give a larger minimum than minimizing over $\mathcal{S}'$ (respectively $\mathcal{S}$); multiplying by $-1/\VM<0$ reverses the inequalities and gives $\gamma^{\mathcal{S}'}_{1,i}\le\gamS{i}\le\gfull{i}$.

\emph{Upper bound.} The joint covariance matrix of $\big(f^{(1)}(Z),f^{(i)}(Z)\big)$,
\[
\begin{pmatrix}
\Sigma_1 & \Sigma_{1,i} \\
\Sigma_{i,1} & \Sigma_i
\end{pmatrix}
\succeq 0,
\]
is positive semidefinite, so its Schur complement with respect to the block $\Sigma_i \succ 0$ satisfies
\[
\Sigma_1 - \Sigma_{1,i} \Sigma_i^{-1} \Sigma_{i,1} \succeq 0.
\]
For $M\succeq 0$, the trace of the product of two positive semidefinite matrices is nonnegative, hence
\[
\tr\!\big(M(\Sigma_1 - \Sigma_{1,i} \Sigma_i^{-1} \Sigma_{i,1})\big)\ge 0,
\qquad\text{i.e.}\qquad
\tr\!\big( M \Sigma_{1,i} \Sigma_i^{-1} \Sigma_{i,1} \big)
\le
\tr(M\Sigma_1) = \VM .
\]
Dividing by $\VM>0$ and using \eqref{eq:gamma-full} yields $\gfull{i}\le 1$.
\end{proof}

\begin{remark}[Reduction to the squared Pearson correlation]
\label{rem:pearson}
For $d=1$ (and any scalar weight $M>0$), all four classes coincide with $\mathbb{R}$, and
\[
\Sigma_1 = \Var(f^{(1)}(Z)), \quad
\Sigma_i = \Var(f^{(i)}(Z)), \quad
\Sigma_{1,i} = \Cov(f^{(1)}(Z), f^{(i)}(Z)),
\]
so that
\[
\gfull{i}
=
\frac{\Cov(f^{(1)}(Z), f^{(i)}(Z))^2}
{\Var(f^{(1)}(Z))\,\Var(f^{(i)}(Z))}
=
\rho_{1,i}^2 .
\]
Thus the effective correlation can be regarded as a generalization of the squared Pearson correlation used in scalar MFMC.
\end{remark}

\subsubsection{Closed forms for the reduced classes}
\label{sec:reduced-classes}

Each reduced class also admits a simple closed-form optimal coefficient when the weighting matrix $M$ is compatible with the imposed structure: block-diagonal $M$ for the block class and diagonal $M$ for the diagonal class, while the scalar class requires only $\tr(M\Sigma_i)>0$. For the block and diagonal classes, the structure makes the objective $\Phi_i$ in \eqref{eq:cv-problem} decouple into independent subproblems of the quadratic form already solved in Lemma~\ref{lem:optimal-full}; for the scalar class it reduces to a one-dimensional quadratic.

\begin{proposition}[Optimal block-diagonal control variate]
\label{prop:block}
Partition the components as $\{1,\dots,d\} = \mathcal{I}_1\cup\cdots\cup \mathcal{I}_B$ with block sizes $d_b=|\mathcal{I}_b|$, and suppose that
\[
M=\blkdiag(M^{[1]},\dots,M^{[B]}),
\qquad
M^{[b]}\in\mathbb{R}^{d_b\times d_b},
\]
is block-diagonal with respect to this partition. Let $\Sigma_i^{[b]}$ and $\Sigma_{1,i}^{[b]}$ denote the $(b,b)$ diagonal blocks of $\Sigma_i$ and $\Sigma_{1,i}$, and assume $\Sigma_i^{[b]}\succ 0$. Then
\[
A_{i}^{\star}
=
\blkdiag\!\big(A_{i}^{\star[1]},\dots,A_{i}^{\star[B]}\big),
\qquad
A_{i}^{\star[b]}
=
\Sigma_{1,i}^{[b]} \big(\Sigma_i^{[b]}\big)^{-1},
\]
minimizes $\Phi_i$ over $\mathcal{S}_{\mathrm{block}}$, and the corresponding effective correlation is
\begin{equation}
\label{eq:gamma-block}
\gblock{i}
=
\frac{1}{\VM}
\sum_{b=1}^B
\tr\!\big(
M^{[b]}
\Sigma_{1,i}^{[b]}
(\Sigma_i^{[b]})^{-1}
\Sigma_{i,1}^{[b]}
\big) .
\end{equation}
\end{proposition}

\begin{proof}
For block-diagonal $A=\blkdiag(A^{[1]},\dots,A^{[B]})$ and $M$, the
off-diagonal blocks of $\Sigma_i$ and $\Sigma_{i,1}$ drop out of the traces,
and the problem decouples into $B$ independent subproblems:
\[
\Phi_i(A)
=
\sum_{b=1}^B
\Big[
\tr\!\big(M^{[b]} A^{[b]} \Sigma_i^{[b]} (A^{[b]})^\top\big)
-2\tr\!\big(M^{[b]} A^{[b]} \Sigma_{i,1}^{[b]}\big)
\Big].
\]
The $b$th summand is exactly the objective of Lemma~\ref{lem:optimal-full} on $\mathbb{R}^{d_b\times d_b}$ with data $M^{[b]}$, $\Sigma_i^{[b]}$, and $\Sigma_{1,i}^{[b]}$, so it is minimized by
\[
A^{\star[b]}=\Sigma_{1,i}^{[b]}(\Sigma_i^{[b]})^{-1},
\quad\text{with minimum value}\quad
-\tr\!\big(M^{[b]}\Sigma_{1,i}^{[b]}(\Sigma_i^{[b]})^{-1}\Sigma_{i,1}^{[b]}\big).
\]
Summing over blocks and dividing by $-\VM$ gives \eqref{eq:gamma-block}.
\end{proof}

Thus each block of HF components is corrected only by the corresponding block of LF components; correlations \emph{within} a block are fully exploited, while correlations \emph{across} blocks are ignored. Only the within-block covariance matrices, of size $d_b\times d_b$, must be estimated and inverted. In applications the partition can be designed based on the problem: components belonging to the same physical field, the same spatial region, or the same coupled subsystem are grouped together.

\begin{proposition}[Optimal diagonal control variate]
\label{prop:diag}
Suppose that $M$ is diagonal, $M=\diag(\omega_1,\dots,\omega_d)$ with $\omega_q\ge 0$, and assume $(\Sigma_i)_{qq}>0$ for all $q$. Then
\[
A_i^\star = \diag(a_{i,1}^\star,\dots,a_{i,d}^\star),
\qquad
a_{i,q}^\star = \frac{(\Sigma_{1,i})_{qq}}{(\Sigma_i)_{qq}},
\quad q=1,\dots,d,
\]
minimizes $\Phi_i$ over $\mathcal{S}_{\mathrm{diag}}$, and the corresponding effective correlation is
\begin{equation}
\label{eq:gamma-diag}
\gdiag{i}
=
\frac{1}{\VM}
\sum_{q=1}^d
\omega_q\,
\frac{(\Sigma_{1,i})_{qq}^2}{(\Sigma_i)_{qq}} .
\end{equation}
\end{proposition}

\begin{proof}
A diagonal matrix is block-diagonal when each block consists of a single component, and a diagonal $M$ is block-diagonal with respect to the same partition. Proposition~\ref{prop:block} with $d_b=1$, $M^{[q]}=\omega_q$, $\Sigma_i^{[q]}=(\Sigma_i)_{qq}$, and $\Sigma_{1,i}^{[q]}=(\Sigma_{1,i})_{qq}$ gives the result.
\end{proof}

Here each component of the HF output is corrected only by the \emph{same} component of the LF output, with its own regression coefficient $a_{i,q}^\star$; cross-component covariance information is not used at all, and only the $2d$ scalars $(\Sigma_i)_{qq}$, $(\Sigma_{1,i})_{qq}$ need to be estimated, and no matrix inversion is required.

\begin{remark}[Connection to componentwise MFMC]
The diagonal estimator recovers the spatially dependent MFMC estimator of Quaglino et al.\ \cite{quaglino2019high}, in which an independent scalar coefficient is attached to each output component and a spatially integrated mean-squared error is minimized; the weights $\omega_q$ play the role of the spatial quadrature weights. The present framework places that construction inside a hierarchy in which it can be compared with more and less expressive coefficient structures.
\end{remark}

\begin{proposition}[Optimal scalar control variate]
\label{prop:scalar}
Assume $\tr(M \Sigma_i) > 0$. Then $A_i^\star=\alpha_i^\star I$ with
\[
\alpha_i^\star
=
\frac{\tr(M \Sigma_{i,1})}{\tr(M \Sigma_i)}
\]
minimizes $\Phi_i$ over $\mathcal{S}_{\mathrm{scalar}}$, and the corresponding effective correlation is
\begin{equation}
\label{eq:gamma-scalar}
\gscalar{i}
=
\frac{
\tr(M \Sigma_{i,1})^2
}{
\VM \, \tr(M \Sigma_i)
} .
\end{equation}
\end{proposition}

\begin{proof}
Substituting $A=\alpha I$ into $\Phi_i$ gives the scalar quadratic
\[
\Phi_i(\alpha I)
=
\alpha^2 \tr(M \Sigma_i)
- 2 \alpha \tr(M \Sigma_{i,1}),
\]
whose leading coefficient $\tr(M\Sigma_i)$ is positive. It is minimized at
\[
\alpha_i^\star = \frac{\tr(M \Sigma_{i,1})}{\tr(M \Sigma_i)},
\quad\text{with minimum value}\quad
-\frac{\tr(M \Sigma_{i,1})^2}{\tr(M \Sigma_i)}.
\]
Dividing by $-\VM$ gives \eqref{eq:gamma-scalar}.
\end{proof}

The scalar class uses a single coefficient shared by all components: it can only rescale the LF correction as a whole, and its effective correlation involves the covariances only through two aggregate traces. It is the cheapest structure to estimate, with one coefficient per fidelity level.

Since the four classes are nested, Lemma~\ref{lem:gamma-properties} immediately orders their effective correlations.

\begin{corollary}[Hierarchy of effective correlations]\label{corollary:Hierarchy of effective correlations}
For every weighting matrix $M\succeq 0$ and each $i=2,\dots,k$,
\[
\gscalar{i}
\le
\gdiag{i}
\le
\gblock{i}
\le
\gfull{i}
\le 1 .
\]
If $M$ is diagonal, all four quantities are given by the closed forms of Lemma~\ref{lem:optimal-full} and Propositions~\ref{prop:block}--\ref{prop:scalar} (a diagonal matrix is block-diagonal with respect to any partition).
\end{corollary}

\begin{proof}
Immediate from Lemma~\ref{lem:gamma-properties} and the inclusions
\[
\mathcal{S}_{\mathrm{scalar}}\subseteq
\mathcal{S}_{\mathrm{diag}}\subseteq
\mathcal{S}_{\mathrm{block}}\subseteq
\mathcal{S}_{\mathrm{full}}.
\]
\end{proof}

At the population level, a more expressive class is therefore never worse.
In practice, however, the covariance quantities that define the optimal
coefficients must be estimated from pilot data, and a more expressive class requires
more sample data. \Cref{tab:class-cost} summarizes, for each class, the
weighting matrices for which the closed form holds, the optimal coefficient,
the number of covariance entries to be estimated for each low-fidelity model,
and the largest matrix to be inverted. In high dimension the last column is
the binding constraint, because a sample covariance computed from $n_{\mathrm{pilot}}$ samples has
rank at most $n_{\mathrm{pilot}}-1$, so the pilot estimate of a $q\times q$
block is singular unless $n_{\mathrm{pilot}}>q$. The full class therefore
requires $n_{\mathrm{pilot}}>d$, whereas the block class only requires
$n_{\mathrm{pilot}}>\max_b d_b$ and the scalar and diagonal classes require
no inversion at all.

\begin{table}[htbp]
\caption{Requirements of the four structural classes. The entries to
estimate are those of $\Sigma_i$ and $\Sigma_{1,i}$ required for each
low-fidelity model $i$; the weighted high-fidelity variance
$\VM=\tr(M\Sigma_1)$ is needed by all classes. The plug-in estimate of the
largest inverse is nonsingular only if $n_{\mathrm{pilot}}$ exceeds its size;
$\bar d:=\max_b d_b$ denotes the largest block size.}
\label{tab:class-cost}
\centering
\small
\setlength{\tabcolsep}{4pt}
\begin{tabular}{@{}lllll@{}}
\toprule
Class & $M$ for closed form & Optimal $A_i^\star$ & Entries to estimate & Largest inverse \\
\midrule
scalar & any & $\dfrac{\tr(M\Sigma_{i,1})}{\tr(M\Sigma_i)}\,I$ & $2$ traces & none \\[1.2em]
diagonal & diagonal & $\diag\!\Big(\dfrac{(\Sigma_{1,i})_{qq}}{(\Sigma_i)_{qq}}\Big)$ & $2d$ & none \\[1.2em]
block & block-diagonal & $\blkdiag\big(\Sigma_{1,i}^{[b]}(\Sigma_i^{[b]})^{-1}\big)$ & $\mathcal{O}\big(\sum_b d_b^2\big)$ & $\bar d\times\bar d$ \\[0.6em]
full & any & $\Sigma_{1,i}\Sigma_i^{-1}$ & $\mathcal{O}(d^2)$ & $d\times d$ \\
\bottomrule
\end{tabular}
\end{table}

\subsection{Optimal sample allocation}
\label{sec:allocation}

We now turn to the second design question: how to distribute the budget $p$ over the fidelity levels. Fix a structural class $\mathcal{S}$ (with the corresponding assumption on $M$ from Section~\ref{sec:reduced-classes}) and abbreviate $\gamma_{1,i} := \gamS{i}$. By Definition~\ref{def:effective-correlation}, optimizing the control variates at fixed sample sizes replaces each $\Phi_i(A_i)$ in \eqref{eq:mse-separable} by its minimum $-\VM\,\gamma_{1,i}$, so the partially optimized MSE is
\[
\mathcal{E}_M^\star(m)
:=
\min_{A_2,\dots,A_k\in\mathcal{S}}
\mathcal{E}_M(\muMF)
=
\frac{\VM}{m_1}
-
\VM
\sum_{i=2}^k
\left(\frac{1}{m_{i-1}}-\frac{1}{m_i}\right)
\gamma_{1,i}.
\]
Collecting the coefficient of each $1/m_i$ (a telescoping regrouping) puts this into the separable form
\begin{equation}
\label{eq:mse-xi}
\mathcal{E}_M^\star(m)
=
\VM \sum_{i=1}^k \frac{\xi_i}{m_i},
\qquad
\xi_i := \gamma_{1,i}-\gamma_{1,i+1},
\quad i=1,\dots,k,
\end{equation}
with the two boundary conventions
\[
\gamma_{1,1}:=1,
\qquad
\gamma_{1,k+1}:=0.
\]
Each coefficient $\xi_i$ measures the variance-reduction gap between two consecutive fidelity levels: it is the share of the weighted HF variance that model $i$ can explain but model $i+1$ cannot, and it is exactly this share that must be paid for with samples at level $i$.

The allocation problem is now identical in form to the scalar one, with $\rho_{1,i}^2$ replaced by $\gamma_{1,i}$, and we solve it in this general form. Since the statement depends on the class $\mathcal{S}$ only through the numbers $\gamma_{1,i}$, a single lemma covers all four structural classes.

\begin{lemma}[Sample allocation for a given effective-correlation sequence]\label{thm:unified_sample_allocation}
Let the effective correlations satisfy the ordering
\[
1=\gamma_{1,1}
>
\gamma_{1,2}
>
\cdots
>
\gamma_{1,k}
>
\gamma_{1,k+1}=0,
\]
so that $\xi_i = \gamma_{1,i}-\gamma_{1,i+1} > 0$ for $i=1,\dots,k$, and assume the cost--correlation condition
\begin{equation}
\label{eq:cost-condition}
\frac{w_i}{w_{i+1}}
\ge
\frac{\xi_i}{\xi_{i+1}}
=
\frac{\gamma_{1,i}-\gamma_{1,i+1}}
{\gamma_{1,i+1}-\gamma_{1,i+2}},
\qquad i=1,\dots,k-1.
\end{equation}
Then the constrained allocation problem
\[
\min_{m_1,\dots,m_k}
\;
\VM\sum_{i=1}^k\frac{\xi_i}{m_i}
\qquad
\text{subject to}
\qquad
\sum_{i=1}^k w_i m_i=p,
\qquad
0<m_1\le \cdots \le m_k,
\]
has the unique solution
\[
m_i^\star = m_1^\star r_i,
\qquad
r_i
=
\sqrt{\frac{w_1\xi_i}{w_i\xi_1}},
\qquad
m_1^\star
=
\frac{p}{\sum_{j=1}^k w_j r_j},
\qquad i=1,\dots,k,
\]
with $r_1=1$. The corresponding optimal MSE is
\begin{equation}
\label{eq:optimal-mse}
\mathcal{E}_M^\star
=
\frac{\VM}{p}
\left(
\sum_{i=1}^k \sqrt{w_i \xi_i}
\right)^{\!2}.
\end{equation}
\end{lemma}

\begin{proof}
Since $\VM>0$ is a constant, it can be dropped from the objective. Consider first the \emph{relaxed} problem obtained by discarding the monotonicity constraints:
\[
\min_{m_i>0}\sum_{i=1}^k \frac{\xi_i}{m_i}
\qquad\text{subject to}\qquad
\sum_{i=1}^k w_i m_i=p.
\]
The objective is strictly convex on the feasible set (each term $\xi_i/m_i$ is strictly convex in $m_i>0$ because $\xi_i>0$) and the constraint is linear, so a stationary point of the Lagrangian
\[
L(m,\lambda)
=
\sum_{i=1}^k \frac{\xi_i}{m_i}
+
\lambda\left(\sum_{i=1}^k w_i m_i-p\right)
\]
is the unique global minimizer. The first-order conditions
\[
-\frac{\xi_i}{m_i^2}+\lambda w_i=0
\qquad\Longrightarrow\qquad
m_i
=
\sqrt{\frac{\xi_i}{\lambda w_i}},
\]
combined with the budget constraint, yield
\[
m_i^\star
=
\frac{
p\sqrt{\xi_i/w_i}
}{
\sum_{j=1}^k \sqrt{\xi_j w_j}
},
\qquad i=1,\dots,k,
\]
which is the stated solution written relative to $m_1^\star$, with
\[
r_i=\frac{m_i^\star}{m_1^\star}=\sqrt{\frac{w_1\xi_i}{w_i\xi_1}} .
\]

It remains to check that the relaxed minimizer is feasible for the original problem, i.e.\ that $m_1^\star\le\cdots\le m_k^\star$. From the closed form,
\[
\frac{m_i^\star}{m_{i+1}^\star}
=
\sqrt{
\frac{\xi_i/w_i}{\xi_{i+1}/w_{i+1}}
},
\]
so the cost--correlation condition \eqref{eq:cost-condition} is equivalent to $m_i^\star\le m_{i+1}^\star$ for $i=1,\dots,k-1$. Hence the relaxed minimizer satisfies the monotonicity constraints; being the unique global minimizer of a problem with a larger feasible set, it is also the unique global minimizer of the original problem.

Finally, substituting $m_i^\star = p\sqrt{\xi_i/w_i}\,\big/\sum_j\sqrt{\xi_j w_j}$ into the objective gives
\[
\sum_{i=1}^k\frac{\xi_i}{m_i^\star}
=
\frac{1}{p}
\left(\sum_{i=1}^k\sqrt{w_i\xi_i}\right)
\left(\sum_{j=1}^k\sqrt{w_j\xi_j}\right)
=
\frac{1}{p}\left(\sum_{i=1}^k \sqrt{w_i \xi_i}\right)^{\!2},
\]
which is \eqref{eq:optimal-mse} after restoring the factor $\VM$.
\end{proof}

Combining the optimal coefficients of Section~\ref{sec:optimal-cv} with Lemma~\ref{thm:unified_sample_allocation} gives the main result of this paper: a complete, closed-form vector-valued MFMC design.

\begin{theorem}[Optimal vector-valued MFMC estimator]
\label{thm:main-vvmfmc}
Fix one of the four structural classes of Section~\ref{sec:optimal-cv},
\[
\mathcal{S}\in\{\mathcal{S}_{\mathrm{scalar}},\,\mathcal{S}_{\mathrm{diag}},\,\mathcal{S}_{\mathrm{block}},\,\mathcal{S}_{\mathrm{full}}\},
\]
together with the corresponding assumption on $M$, and let $\gamma_{1,i}=\gamS{i}$ be the effective correlations of Definition~\ref{def:effective-correlation}, computed in closed form in \eqref{eq:gamma-full} and \eqref{eq:gamma-block}--\eqref{eq:gamma-scalar}. Assume the models are ordered so that
\[
1=\gamma_{1,1}>\gamma_{1,2}>\cdots>\gamma_{1,k}>\gamma_{1,k+1}=0,
\]
and that the costs satisfy the cost--correlation condition \eqref{eq:cost-condition}. Then, over all coefficient matrices $A_2,\dots,A_k\in\mathcal{S}$ and all sample sizes satisfying $0<m_1\le\cdots\le m_k$ and the budget constraint $\sum_{i=1}^k w_i m_i = p$, the weighted MSE $\mathcal{E}_M(\muMF)$ is minimized by
\begin{itemize}
\item the optimal coefficients $A_i^\star$ of Lemma~\ref{lem:optimal-full} (full class) or Propositions~\ref{prop:block}--\ref{prop:scalar} (reduced classes), and
\item the sample allocation
\[
m_i^\star = m_1^\star r_i,
\qquad
r_i
=
\sqrt{
\frac{
w_1\bigl(
\gamma_{1,i}
-
\gamma_{1,i+1}
\bigr)
}{
w_i\bigl(
1-\gamma_{1,2}
\bigr)
}
},
\qquad
m_1^\star
=
\frac{p}{\sum_{j=1}^k w_j r_j}.
\]
\end{itemize}
The resulting optimal MSE is given by \eqref{eq:optimal-mse} with $\xi_i=\gamma_{1,i}-\gamma_{1,i+1}$.
\end{theorem}

\begin{proof}
The minimization proceeds in two stages, because the optimal coefficients do not depend on the sample sizes. For any fixed feasible $m$, Corollary~\ref{corollary: MSE matrix cv} expresses $\mathcal{E}_M$ as $\VM/m_1$ plus a sum of terms $\delta_i\Phi_i(A_i)$ with $\delta_i\ge 0$, which are minimized termwise by the class-optimal coefficients $A_i^\star$; by Definition~\ref{def:effective-correlation} the resulting value is $\mathcal{E}_M^\star(m)$ in \eqref{eq:mse-xi}. Minimizing $\mathcal{E}_M^\star(m)$ over the feasible sample sizes is exactly the problem solved by Lemma~\ref{thm:unified_sample_allocation}.
\end{proof}

\begin{remark}[Exact analogy with scalar MFMC]
Theorem~\ref{thm:main-vvmfmc} has the same structure as the scalar result (Theorem~\ref{thm:scalar-mfmc}): the allocation ratios, the cost--correlation ordering condition, and the model-selection logic of \cite{peherstorfer2016optimal} all carry over directly once $\rho_{1,i}^2$ is replaced by $\gamma_{1,i}$. In particular, for $d=1$ the two theorems coincide (Remark~\ref{rem:pearson}). The model-management procedure developed for scalar MFMC, including the model-selection step that searches over subsets of models, can therefore be generalized for vector-valued outputs.
\end{remark}

\subsection{The price of structural reduction}
\label{sec:reduction-loss}

The hierarchy of Corollary~\ref{corollary:Hierarchy of effective correlations} states that more general coefficient classes achieve larger effective correlations. The next result shows that this ordering carries over to the MSE.

\begin{corollary}[Variance-reduction loss induced by reduced control variates]
\label{cor:reduction-loss}
Let $\mathcal{S}_{\mathrm{red}}\subseteq\mathbb{R}^{d\times d}$ be a structural class containing $0$. For a class $\mathcal{S}$ and a nested allocation $m$ with $0<m_1\le\cdots\le m_k$, let
\[
\mathcal{E}^{\mathcal{S}}_M(m):=\min_{A_2,\dots,A_k\in\mathcal{S}}\mathcal{E}_M(\muMF),
\]
and let
\[
\mathcal{E}^{\mathcal{S},\star}_M
:=
\min\Bigl\{\mathcal{E}^{\mathcal{S}}_M(m):\;
0<m_1\le\cdots\le m_k,\ \textstyle\sum_{i=1}^k w_i m_i=p\Bigr\}
\]
denote the MSE when both the coefficients and the allocation are optimized for $\mathcal{S}$.
\begin{enumerate}
\item[(a)] With each class using its own optimal allocation,
\[
\mathcal{E}^{\mathrm{full},\star}_M \le \mathcal{E}^{\mathrm{red},\star}_M .
\]
\item[(b)] For a fixed nested allocation $m$, with $\delta_i=1/m_{i-1}-1/m_i\ge 0$,
\[
\mathcal{E}^{\mathrm{red}}_M(m)-\mathcal{E}^{\mathrm{full}}_M(m)
=
\VM
\sum_{i=2}^k
\delta_i
\left(
\gfull{i}
-
\gamma^{\mathcal{S}_{\mathrm{red}}}_{1,i}
\right) \;\ge\; 0 .
\]
\end{enumerate}
\end{corollary}

\begin{proof}
(a) By definition,
$\mathcal{E}^{\mathcal{S},\star}_M=\min_{m}\min_{A_2,\dots,A_k\in\mathcal{S}}\mathcal{E}_M(\muMF)$,
where $m$ ranges over the same set of feasible allocations for every class.
Since $\mathcal{S}_{\mathrm{red}}\subseteq\mathcal{S}_{\mathrm{full}}$, the minimization
for $\mathcal{S}_{\mathrm{red}}$ is over a subset of the one for $\mathcal{S}_{\mathrm{full}}$,
which gives $\mathcal{E}^{\mathrm{full},\star}_M\le\mathcal{E}^{\mathrm{red},\star}_M$.

(b) By Corollary~\ref{corollary: MSE matrix cv} and Definition~\ref{def:effective-correlation}, optimizing the coefficients within a class $\mathcal{S}$ at the fixed allocation $m$ gives
\[
\mathcal{E}^{\mathcal{S}}_M(m)
=
\frac{\VM}{m_1}
-
\VM
\sum_{i=2}^k
\delta_i\,\gamS{i}.
\]
Subtracting this identity for $\mathcal{S}=\mathcal{S}_{\mathrm{full}}$ from the same identity for $\mathcal{S}=\mathcal{S}_{\mathrm{red}}$ yields the stated formula. Every term is nonnegative because $\delta_i\ge 0$ for nested sample sizes and $\gamma^{\mathcal{S}_{\mathrm{red}}}_{1,i}\le\gfull{i}$ by Lemma~\ref{lem:gamma-properties}.
\end{proof}

Part~(a) shows that the full class is never worse than a reduced class, even when each class uses its own optimal allocation. Part~(b) quantifies the population-level cost of imposing reduced control variate matrix structure at a fixed allocation: the loss at level $i$ is proportional to the correlation gap $\gfull{i}-\gamma^{\mathcal{S}_{\mathrm{red}}}_{1,i}$, weighted by the sample allocation term $\delta_i$.
\subsection{Linearly mapped quantities of interest}
\label{sec:special-cases}

In many applications each model returns a high-dimensional state, but the
quantity of interest is a linear map of it,
\[
g^{(i)}(Z)=Bf^{(i)}(Z)\in\mathbb{R}^r,
\qquad
q:=\E[g^{(1)}(Z)]=B\mu,
\qquad
B\in\mathbb{R}^{r\times d},
\]
for example forces extracted from a pressure field, pointwise observations,
or averages over prescribed regions. An MFMC estimator of $q$ can either
\emph{project first}, building the control variates from the mapped outputs
$g^{(i)}$ only, or \emph{correct first}, building them from the full LF
states and applying $B$ only to the final estimate. For any estimator
$\widehat\mu$ of $\mu$,
\[
\E\|B\widehat\mu-q\|_2^2=\E\|\widehat\mu-\mu\|_{B^\top B}^2,
\]
so correcting first is the weighted problem of the previous sections with
$M=B^\top B$. The next result shows that nothing is lost by this
construction.

\begin{corollary}[Linearly mapped QoI]
\label{cor:linear-qoi}
For estimators
\[
\widehat q=B\bar y^{(1)}_{m_1}+\sum_{i=2}^k K_i\bigl(\bar y^{(i)}_{m_i}-\bar y^{(i)}_{m_{i-1}}\bigr),
\qquad K_i\in\mathbb{R}^{r\times d},
\]
the optimal coefficients are
$K_i^\star=B\Sigma_{1,i}\Sigma_i^{-1}=BA_i^\star$ for every sample
allocation. Hence minimizing $\E\|\widehat q-q\|_2^2$ over the $K_i$ and the
allocation is equivalent to minimizing $\mathcal{E}_M(\muMF)$ with
$M=B^\top B$ over the $A_i$ and the allocation.
\end{corollary}

\begin{proof}
As in Lemmas~\ref{lem:cov_vector_mfmc} and~\ref{lem:mse-trace},
\[
\E\|\widehat q-q\|_2^2
=\frac{\tr(B\Sigma_1B^\top)}{m_1}
+\sum_{i=2}^k\delta_i\Bigl[\tr(K_i\Sigma_iK_i^\top)-2\tr(B\Sigma_{1,i}K_i^\top)\Bigr].
\]
Each bracket is a convex quadratic in $K_i$, minimized where
$K_i\Sigma_i=B\Sigma_{1,i}$, i.e., at $K_i^\star=BA_i^\star$. For $K_i=BA_i$
the bracket equals $\Phi_i(A_i)$ with $M=B^\top B$, and
$\tr(B\Sigma_1B^\top)=\tr(M\Sigma_1)$, which gives the equivalence.
\end{proof}

The most common case is a scalar QoI, $B=c^\top$ with $c\in\mathbb{R}^d$,
so that $q=c^\top\mu$ and $M=cc^\top$. Standard scalar MFMC applied to the
projected outputs $g^{(i)}=c^\top f^{(i)}$ uses one scalar coefficient per
level, $\widehat q_{\mathrm{scalar}}=\bar g^{(1)}_{m_1}+\sum_{i=2}^k\alpha_i
(\bar g^{(i)}_{m_i}-\bar g^{(i)}_{m_{i-1}})$. Writing
$\Delta_i:=\bar y^{(i)}_{m_i}-\bar y^{(i)}_{m_{i-1}}$ and using
$\bar g^{(i)}_m=c^\top\bar y^{(i)}_m$, projecting first and correcting first
(Corollary~\ref{cor:linear-qoi}) give
\begin{align}
\widehat q_{\mathrm{scalar}}
&=c^\top\Bigl[\bar y^{(1)}_{m_1}+\sum_{i=2}^k(\alpha_i I)\,\Delta_i\Bigr],
\label{eq:project-first-scalar-mfmc}\\
\widehat q_{\mathrm{full}}
&=c^\top\Bigl[\bar y^{(1)}_{m_1}+\sum_{i=2}^k A_i^\star\,\Delta_i\Bigr],
\qquad A_i^\star=\Sigma_{1,i}\Sigma_i^{-1}.
\label{eq:scalar-qoi-full-state-estimator}
\end{align}
The two estimators differ only in the control-variate matrix inside the
brackets: $\alpha_iI\in\mathcal{S}_{\mathrm{scalar}}$ for projecting first and
$A_i^\star\in\mathcal{S}_{\mathrm{full}}$ for correcting first.
Under $M=cc^\top$ the optimal coefficient of the scalar class is the scalar
MFMC coefficient. Since $\mathcal{S}_{\mathrm{scalar}}\subseteq\mathcal{S}_{\mathrm{full}}$,
Lemma~\ref{lem:gamma-properties} and Corollary~\ref{cor:reduction-loss}
with $\mathcal{S}_{\mathrm{red}}=\mathcal{S}_{\mathrm{scalar}}$ give the
following result.

\begin{corollary}[Benefit of full-state control variates for a scalar QoI]
\label{cor:full-state-benefit}
Let $\widehat q_{\mathrm{full}}$ be~\eqref{eq:scalar-qoi-full-state-estimator}
and $\widehat q_{\mathrm{scalar}}$ be~\eqref{eq:project-first-scalar-mfmc}
with the optimal scalar MFMC coefficients.
\begin{enumerate}
\item[(a)] If each estimator uses its own optimal sample allocation under
the budget $p$, then
\[
\min_{m}\E[(\widehat q_{\mathrm{full}}-q)^2]\le\min_{m}\E[(\widehat q_{\mathrm{scalar}}-q)^2].
\]
\item[(b)] For a fixed nested allocation $0<m_1\le\cdots\le m_k$,
\begin{equation}
\label{eq:full-state-scalar-qoi-gap}
\E\bigl[(\widehat q_{\mathrm{scalar}}-q)^2\bigr]
-\E\bigl[(\widehat q_{\mathrm{full}}-q)^2\bigr]
=\VM\sum_{i=2}^k\delta_i\bigl(\gfull{i}-\gscalar{i}\bigr)\ge0,
\end{equation}
where $\VM=c^\top\Sigma_1c$ and the effective correlations are evaluated
with $M=cc^\top$.
\end{enumerate}
\end{corollary}

The gap~\eqref{eq:full-state-scalar-qoi-gap} is the correlation information
lost in the projection: $\gscalar{i}$ is the squared correlation left after
both outputs are compressed to the scalar pair
$(c^\top f^{(1)},c^\top f^{(i)})$, whereas $\gfull{i}$ measures how well the
entire LF state predicts the same scalar HF quantity. When each model
evaluation produces the full state anyway, correcting first requires no
additional samples. Its price is that $A_i^\star$ involves $\Sigma_i^{-1}$,
which must be estimated from pilot data.

\subsection{Practical implementation}
\label{sec:implementation}

The design formulas of Theorem~\ref{thm:main-vvmfmc} involve covariance quantities that are unknown in practice. A practical implementation therefore proceeds in two stages. In a \emph{pilot} stage, all $k$ models are evaluated on a common set of $n_{\mathrm{pilot}}$ samples, from which the covariance quantities required by the chosen structural class are estimated. The plug-in coefficients $\widehat A_i$ and effective correlations $\widehat\gamma_{1,i}$ then determine the allocation of the remaining \emph{production} budget. Because the $\widehat\gamma_{1,i}$ carry sampling noise, the ordering and cost--correlation conditions of Lemma~\ref{thm:unified_sample_allocation} may fail for the raw estimates. As in scalar MFMC \cite{peherstorfer2016optimal}, we therefore select the subset of LF models that satisfies these conditions and has the smallest predicted MSE, where the HF-only estimator is included as a candidate. Algorithm~\ref{alg:pilot_reduced_mfmc} summarizes the procedure.

\begin{algorithm}[ht!]
\caption{Practical implementation of vector-valued MFMC}
\label{alg:pilot_reduced_mfmc}
\begin{algorithmic}[1]
\REQUIRE Budget $p$, pilot budget $p_{\mathrm{pilot}}<p$, costs $\{w_i\}_{i=1}^k$,
models $\{f^{(i)}\}_{i=1}^k$, weighting matrix $M$, structural class
$\mathcal{S}\in\{\mathcal{S}_{\mathrm{scalar}},\mathcal{S}_{\mathrm{diag}},
\mathcal{S}_{\mathrm{block}},\mathcal{S}_{\mathrm{full}}\}$.
\ENSURE Estimate $\widehat{\mu}_{\mathrm{MF}}$ of $\mu=\E[f^{(1)}(Z)]$.

\STATE \textbf{Pilot.} Set $n_{\mathrm{pilot}}=\lfloor p_{\mathrm{pilot}}/\sum_{i=1}^k w_i\rfloor$
and $p_{\mathrm{prod}}=p-p_{\mathrm{pilot}}$. Evaluate all $k$ models on
$n_{\mathrm{pilot}}$ common i.i.d.\ samples and form the sample covariances
required by $\mathcal{S}$, together with $\widehat V_M=\tr(M\widehat\Sigma_1)$.

\STATE \textbf{Coefficients.} For $i=2,\dots,k$, compute $\widehat A_i$ and
$\widehat\gamma_{1,i}$ from the plug-in versions of
Lemma~\ref{lem:optimal-full} or Propositions~\ref{prop:block}--\ref{prop:scalar}.

\STATE \textbf{Model selection.} Among the subsets of LF models whose
$\widehat\gamma_{1,i}$ are strictly decreasing and satisfy the cost--correlation
condition~\eqref{eq:cost-condition}, including the HF-only subset, choose the one
with the smallest predicted MSE~\eqref{eq:optimal-mse}, and relabel the selected
models as $f^{(1)},\dots,f^{(k)}$.

\STATE \textbf{Allocation.} For the selected models, compute
$\widehat m_i = \widehat m_1 \widehat r_i$ from
Lemma~\ref{thm:unified_sample_allocation} with budget $p_{\mathrm{prod}}$,
and round to integers $1\le m_1\le\cdots\le m_k$ with
$\sum_i w_i m_i\le p_{\mathrm{prod}}$.

\STATE \textbf{Production.} Draw $m_k$ new i.i.d.\ samples independent of the
pilot, and evaluate model $i$ on the first $m_i$ of them.

\RETURN
$\widehat{\mu}_{\mathrm{MF}}
= \bar y^{(1)}_{m_1}
+ \sum_{i=2}^k \widehat A_i\bigl(\bar y^{(i)}_{m_i}-\bar y^{(i)}_{m_{i-1}}\bigr)$.
\end{algorithmic}
\end{algorithm}

\paragraph{Pilot reuse.}
\label{sec:pilot-reuse}
The $n := n_{\mathrm{pilot}}$ pilot evaluations of all $k$ models are only
used to estimate the parameters. They can also be used in the final
estimator at no additional cost, by placing the pilot samples in front of the
production samples of every model, so that each sample mean is taken over
$n+m_i$ instead of $m_i$ samples, with the design unchanged. The estimator is
then in general no longer unbiased. Without reuse, the design is fixed given
the pilot and the production samples are independent of it, so each
correction term has zero conditional mean,
$\E[\widehat A_i(\bar y^{(i)}_{m_i}-\bar y^{(i)}_{m_{i-1}}) \mid \text{pilot}] = 0$.
With reuse, the correction terms also contain the pilot sample mean
$\bar p^{(i)}$ of model $i$, from which $\widehat A_i$ (and the allocation) is
itself estimated, and
\begin{equation}
\label{eq:reuse-bias}
\E\bigl[\widehat A_i\,\bar p^{(i)}\bigr]
= \E\bigl[\widehat A_i\bigr]\,\mu_i + \Cov\bigl(\widehat A_i,\bar p^{(i)}\bigr),
\end{equation}
where
$\Cov(\widehat A_i,\bar p^{(i)}) := \E[(\widehat A_i-\E\widehat A_i)(\bar p^{(i)}-\mu_i)]$
is in general nonzero. This covariance term makes the estimator biased.
Section~\ref{sec:exp-reuse} compares the resulting MSE with that of the
estimator without reuse.

\section{Numerical Experiments}
\label{sec:numerics}

In this section we conduct three numerical experiments, each consisting of
$N_{\mathrm{rep}}$ independent repetitions of
Algorithm~\ref{alg:pilot_reduced_mfmc}, following the procedure described below.
The total budget is split as $p=p_{\mathrm{pilot}}+p_{\mathrm{prod}}$.
The pilot stage evaluates all $k$ models on
$n_{\mathrm{pilot}}=\lfloor p_{\mathrm{pilot}}/\sum_{i=1}^k w_i\rfloor$ common
samples, from which each coefficient class estimates its covariance
quantities, control variates, and effective correlations; the
production allocation then follows from
Lemma~\ref{thm:unified_sample_allocation} with the plug-in correlations and
the remaining budget $p_{\mathrm{prod}}$. The model-selection step uses an LF
model only if the ordering condition~\eqref{eq:cost-condition} holds and the
resulting predicted MSE~\eqref{eq:optimal-mse} is smaller than that of
HF-only sampling; otherwise the production budget is spent on HF samples,
and the pilot cost is still paid. Unless stated otherwise, the block and full
classes add a small ridge ($10^{-8}$; see the parameter tables) before
inverting a pilot-estimated LF covariance, within each block for the block
class. Within a repetition, all estimators share the same pilot sample.
HF-only Monte Carlo (MC) spends the full budget $p$ on $\lfloor p/w_1\rfloor$
HF samples.

For every estimator we report the empirical MSE, i.e.\ the average of
$\|\hat\mu-\mu\|_M^2$ over the $N_{\mathrm{rep}}$ repetitions, together with its
standard error, and its ratio to the empirical
MSE of HF-only MC. The predicted ratio is the population-optimal
MSE~\eqref{eq:optimal-mse}, evaluated with the population (or reference)
effective correlations and the production budget $p_{\mathrm{prod}}$, divided
by the HF-only MSE $\VM w_1/p$. When the population
covariances are not available in closed form, the reference mean and the
reference effective correlations are computed from large HF and LF samples
that are independent of the repetitions and whose Monte Carlo error is
negligible relative to the reported MSEs.

\subsection{Experiment 1: Gaussian test problem}
\label{sec:experiment-gaussian}

We first consider a Gaussian test problem, in which all population
covariances, and hence all effective correlations, are available in closed
form, so that the pilot-estimated estimators can be compared with the
theory. The HF output is
\[
X := f^{(1)}(Z) \sim \mathcal{N}(\mu, \Sigma_{\mathrm{HF}}),
\qquad
\Sigma_{\mathrm{HF}} = D^{1/2} R\, D^{1/2},
\qquad
R_{rs} = \exp\!\big(-|r-s|/\ell_{\mathrm{HF}}\big),
\]
where the entries of $\mu$ increase linearly from $\mu_{\min}$ to
$\mu_{\max}$ and the marginal variances on the diagonal of $D$ decrease
linearly from $\lambda_{\max}$ to $\lambda_{\min}$. The correlation length
$\ell_{\mathrm{HF}}$ controls the cross-component correlation, which is
strongest between neighboring components; contiguous blocks of adjacent
components are designed to exploit this structure. Each LF model is
\[
f^{(i)}(Z) = T_i X + \varepsilon_i,
\qquad
\varepsilon_i \sim \mathcal{N}\big(0,\,
\sigma_{\varepsilon_i}^2[(1-\rho_{\varepsilon_i}) I_d
+ \rho_{\varepsilon_i}\mathbf{1}\mathbf{1}^\top]\big),
\qquad i = 2, \dots, k,
\]
with a diagonal gain matrix $T_i$ and independent equicorrelated noise, so
that $\Sigma_i = T_i \Sigma_{\mathrm{HF}} T_i + \Cov(\varepsilon_i)$,
$\Sigma_{1,i} = \Sigma_{\mathrm{HF}} T_i$, and
$\Sigma_{i,j} = T_i \Sigma_{\mathrm{HF}} T_j$ for $i \neq j$. A uniform gain
$T_i = \theta_i I_d$ models a uniformly attenuated surrogate, and a
nonconstant diagonal a surrogate whose local bias varies across components
(Case~4).

The block partition is contiguous, with a common block size matched to the
pilot budget so that each block covariance is estimated from more samples
than its dimension, and $M = I_d$ in Cases~1--3.
Table~\ref{tab:gaussian-params} collects the baseline parameters; with
them, $n_{\mathrm{pilot}} = \lfloor 4000/51\rfloor = 78$ and HF-only MC uses
$160$ HF samples.

\begin{table}[htbp]
\caption{Parameters of the Gaussian test problem and their baseline (Case~1)
values. The subsequent cases modify only the entries stated in the text.}
\label{tab:gaussian-params}
\centering
\small
\begin{tabular}{@{}llc@{}}
\toprule
Parameter & Meaning & Case 1 value \\
\midrule
$p$ & total computational budget & 8000 \\
$p_{\mathrm{pilot}}$ & pilot budget & 4000 \\
$w_1, w_2$ & model costs & 50, 1 \\
$d$ & dimension of the QoI & 10 \\
$\mu_{\min}, \mu_{\max}$ & range of the true mean entries & 1, 2 \\
$T_2$ & LF gain matrix & $0.95\, I_d$ \\
$\lambda_{\max}, \lambda_{\min}$ & range of HF marginal variances & 5, 0.5 \\
$\ell_{\mathrm{HF}}$ & HF cross-component correlation length & 20 \\
$\sigma_{\varepsilon}$ & LF noise standard deviation & 0.5 \\
$\rho_{\varepsilon}$ & LF noise equicorrelation & 0.2 \\
block partition & contiguous blocks (block class) & $2 \times 5$ \\
$N_{\mathrm{rep}}$ & number of independent repetitions & 500 \\
$\tau$ & ridge regularization for $\hat\Sigma_i^{-1}$ (block/full CV) & $10^{-8}$ \\
\bottomrule
\end{tabular}
\end{table}

\subsubsection*{Case 1: baseline parameter setting ($d = 10$)}

In the baseline case the LF model is strongly correlated with the HF model
($T_2 = 0.95\,I_d$), the correlation length $\ell_{\mathrm{HF}} = 20$ sets
strong correlation across the $d = 10$ output components, half of the
budget is spent on the pilot ($n_{\mathrm{pilot}} = 78$), and the block class
uses two contiguous blocks of five components. The true effective
correlations are
\[
\gscalar{2} = 0.908488,
\quad
\gdiag{2} = 0.912202,
\quad
\gblock{2} = 0.939843,
\quad
\gfull{2} = 0.943036,
\]
ordered as guaranteed by
Corollary~\ref{corollary:Hierarchy of effective correlations}.

\begin{table}[htbp]
\caption{Experiment~1: empirical MSE $\pm$ one standard error over
$N_{\mathrm{rep}} = 500$ repetitions, ratio to the empirical MSE of HF-only
MC, and the population-predicted ratio~\eqref{eq:optimal-mse}, which
includes the pilot-cost factor $p/p_{\mathrm{prod}} = 2$. Cases~1--3 use
$M = I_d$; Case~4 estimates the scalar QoI $q = c^\top\mu$ with
$M = cc^\top$, where ``projected'' is scalar MFMC on
$g^{(i)} = c^\top f^{(i)}$ and ``full state'' is the full-state estimator of
Corollary~\ref{cor:full-state-benefit}.}
\label{tab:exp1}
\centering
\small
\input{tables/exp1.tex}
\end{table}

\begin{figure}[htbp]
\centering
\includegraphics[width=\textwidth]{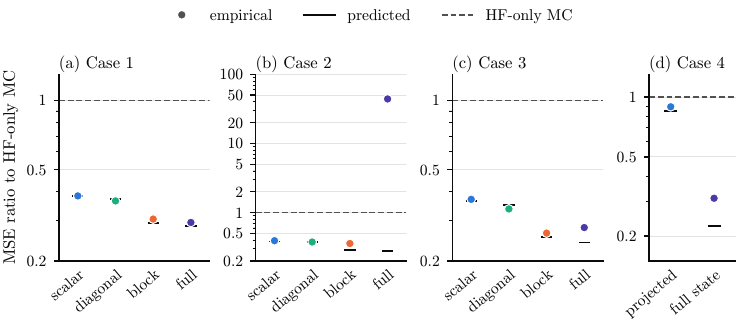}
\caption{Experiment~1: ratio of the empirical MSE of each MFMC estimator to
that of HF-only MC (log scale); black bars mark the population-predicted
ratio. (a)--(c)~Cases~1--3 ($M = I_d$); in Case~2
($d = 100 > n_{\mathrm{pilot}} = 78$) the plug-in full-matrix estimator
fails. (d)~Case~4, scalar QoI $q = c^\top\mu$ ($M = cc^\top$).}
\label{fig:exp1}
\end{figure}

All four MFMC estimators reduce the MSE substantially
(Table~\ref{tab:exp1} and Figure~\ref{fig:exp1}(a)), and their ranking
follows the hierarchy. The two classes that use cross-component correlation
(block and full) clearly outperform the two
componentwise classes.
Since $n_{\mathrm{pilot}} = 78 \gg d = 10$, even the full $10 \times 10$
covariance is estimated accurately, and the sample allocations
closely track the optimal design.

\subsubsection*{Case 2: increasing the dimension to $d = 100$}

The second case increases the dimension to $d = 100$ and keeps all other
parameters, to test the classes when the dimension exceeds the pilot sample
size; the block class uses five contiguous blocks of twenty components. The true
effective correlations,
\[
\gscalar{2} = 0.908488,
\quad
\gdiag{2} = 0.911237,
\quad
\gblock{2} = 0.942180,
\quad
\gfull{2} = 0.944629,
\]
again respect the hierarchy of
Corollary~\ref{corollary:Hierarchy of effective correlations} and place the
full matrix first. However, the empirical behavior is different
(Table~\ref{tab:exp1} and Figure~\ref{fig:exp1}(b)).

The scalar and diagonal estimators behave as in Case~1, but the
full-matrix estimator has an MSE $44$ times larger than HF-only MC.
The failure is caused by unstable pilot estimation: the
pilot-estimated $100 \times 100$ LF covariance has rank at most
$77 < d$, so the effective correlation is estimated as
$\hat\gamma^{\mathrm{full}}_{1,2} \approx 1$, and the allocation assigns
almost the entire production budget to the LF model and leaves a single HF
sample.
The block class obtains the lowest MSE and avoids this instability, since
each $20 \times 20$ block covariance is estimated from $78 > 20$ samples and
remains full rank. With the block size matched to the pilot budget and to the correlation
structure, the block class thus keeps part of the cross-component gain
while remaining as stable as the componentwise classes.

\subsubsection*{Case 3: a third fidelity level ($k = 3$, $d = 10$)}

The third case extends Case~1 by a second, cheaper and less accurate
surrogate,
\[
f^{(3)}(Z) = T_3 X + \varepsilon_3,
\qquad
T_3 = 0.80\, I_d,
\qquad
\sigma_{\varepsilon_3} = 1.0,
\quad
\rho_{\varepsilon_3} = 0.2,
\quad
w_3 = 0.2,
\]
while model 2 and all other parameters are unchanged
($n_{\mathrm{pilot}} = 78$, two blocks of five). The population effective
correlations obey the hierarchy at both fidelity levels,
\[
\begin{aligned}
&\gscalar{2} = 0.9085,\;
\gdiag{2} = 0.9122,\;
\gblock{2} = 0.9398,\;
\gfull{2} = 0.9430, \\
&\gscalar{3} = 0.6377,\;
\gdiag{3} = 0.6664,\;
\gblock{3} = 0.7916,\;
\gfull{3} = 0.8112,
\end{aligned}
\]
and for every class they satisfy the ordering condition of
Lemma~\ref{thm:unified_sample_allocation}, so the closed-form three-level
allocation applies. The gaps between the classes are wider at the less
accurate third level, where the cross-component structure carries a larger
share of the recoverable variance.

The cheap third model lowers the predicted ratio of every class.
Empirically, the MSE decreases for the diagonal and block classes and is
almost unchanged for the scalar and full classes
(Table~\ref{tab:exp1} and Figure~\ref{fig:exp1}(c)). The block class now
reaches a slightly lower MSE than the full class, although the full class has the larger population correlation at
both levels. This is consistent with the larger estimation cost of the full
class, which must estimate two $10 \times 10$ LF covariances and the
corresponding cross-covariances from the same $78$ pilot samples.

\subsubsection*{Case 4: scalar QoI, full-state versus scalar control
variates}

The final case tests Corollary~\ref{cor:full-state-benefit}. The models
still produce the full state ($d = 10$), but the reported quantity is the
component mean $q = c^\top \mu$ with $c = \mathbf{1}/d$. We compare the two estimators of
Corollary~\ref{cor:full-state-benefit}: scalar MFMC on the projected outputs
$g^{(i)} = c^\top f^{(i)}$, and the full-state estimator. Both use the budgets of
Table~\ref{tab:gaussian-params} and the same pilot samples. The gain of
correcting first is governed by $\gfull{2} - \gscalar{2}$, the correlation
lost by the projection, so we use an LF model in which this loss
is substantial: its componentwise gains vary in magnitude and sign,
\[
T_2 = \operatorname{diag}\big(\operatorname{linspace}(1.2, -0.8, d)\big),
\qquad
\sigma_\varepsilon = 0.5,
\quad
\rho_\varepsilon = 0.2.
\]
This models a surrogate that tracks the HF state locally but over-predicts
some components and has the wrong sign on others, e.g., a regionally biased
solver. In the projection $c^\top f^{(2)}$ the positively and negatively
correlated contributions cancel, and no single scalar coefficient can undo
this ($\gscalar{2} = 0.715$), whereas the full control vector reweights each
component and corrects its sign before averaging ($\gfull{2} = 0.962$).

Table~\ref{tab:exp1} and Figure~\ref{fig:exp1}(d) support
Corollary~\ref{cor:full-state-benefit}: scalar MFMC on the projected outputs
barely improves on HF-only MC, whereas the full-state estimator reduces its
MSE by a factor of $2.9$.

\subsection{Experiment 2: advection--diffusion equation with a POD ROM}
\label{sec:exp2}

The second experiment applies the vector-valued MFMC estimator to a
PDE model pair: a finite-difference high-fidelity model of a steady
two-dimensional advection--diffusion problem, and a projection-based
reduced-order model (ROM) serving as the low-fidelity model. As is typical
for ROM-based hierarchies, the LF output covariance has low rank, so the full
coefficient class can be implemented without forming a $d \times d$
covariance matrix (Proposition~\ref{prop:rom-reduction}).

\subsubsection{Model problem and the two fidelities}

On the unit square $\Omega = (0,1)^2$ we consider the steady
advection--diffusion equation for a temperature field $T(x,y)$,
\begin{equation}
  U_1\, \frac{\partial T}{\partial x} - \kappa\, \Delta T = 0
  \qquad \text{in } \Omega,
  \label{eq:advdiff}
\end{equation}
with advection speed $U_1$ in the positive $x$ direction and diffusivity
$\kappa$. The left boundary carries a Dirichlet condition that models a hot
inlet band $y \in [\frac13, \frac23]$ with a two-lobe temperature profile
centered at $y = \bar y$, and the remaining three sides carry homogeneous
Neumann conditions. The random input is the parameter vector
$Z = (U_1, \kappa, \bar y)$ with independent uniform components,
\begin{equation*}
  U_1 \sim \mathcal U(0,\,0.5), \qquad
  \kappa \sim \mathcal U(5\times 10^{-4},\,0.025), \qquad
  \bar y \sim \mathcal U(0.4,\,0.6).
\end{equation*}

The HF model discretizes~\eqref{eq:advdiff} by first-order upwind and
centered finite differences on a uniform $75 \times 75$ grid and returns the
full interior state, $f^{(1)}(Z) \in \mathbb{R}^d$ with
$d = 73^2 = 5329$. The LF model is a POD--Galerkin ROM of rank $r = 10$,
constructed once offline from $20$ HF snapshots and then frozen. It returns
the reconstructed state $f^{(2)}(Z) = V q(Z)$, where
$V \in \mathbb{R}^{d \times r}$ has orthonormal columns and
$q(Z) \in \mathbb{R}^r$ are the reduced coordinates; its mean relative
state error is about $0.6\%$. We use the nominal costs $w_1 = 1$ and
$w_2 = 1/140$ in HF-solve equivalents. The offline cost (about $25$ HF
solves) is treated as a fixed model-construction cost external to the
sampling budget.

Table~\ref{tab:exp2params} collects the parameters. The total budget is
$p = 200$ HF-solve equivalents, and we report two pilot sizes,
$n_{\mathrm{pilot}} = 49$ and $n_{\mathrm{pilot}} = 99$ (a quarter and a half
of the budget), to expose the dependence of each class on pilot quality. For the full class
the ridge is applied to the pilot estimate of $\operatorname{Cov}(q)$. The
reference mean is computed from $2^{18}$ HF solves, and the reference
effective correlations from an independent sample of $10^4$ evaluations of
both models.

\begin{table}[htbp]
\caption{Parameters of the advection--diffusion experiment.}
\label{tab:exp2params}
\centering
\small
\begin{tabular}{@{}llc@{}}
\toprule
Parameter & Meaning & Value \\
\midrule
    $n_x \times n_y$ & HF grid ($d = 5329$ interior unknowns) & $75 \times 75$ \\
    $r$ & POD rank of the LF ROM & $10$ \\
    $n_{\mathrm{train}}$ & offline POD training snapshots & $20$ \\
    $w_1,\, w_2$ & nominal model costs (measured $8.3$\,ms, $0.06$\,ms) & $1,\ 1/140$ \\
    $p$ & total budget (HF-solve equivalents) & $200$ \\
    $n_{\mathrm{pilot}}$ & common pilot samples & $49$ / $99$ \\
    $p_{\mathrm{prod}}$ & production budget per estimator & $150$ / $100$ \\
    $m$ & HF-only MC baseline samples (full budget) & $200$ \\
    $N_{\mathrm{rep}}$ & independent repetitions & $500$ \\
    $\lambda$ & relative ridge on $\widehat{\operatorname{Cov}}(q)$ & $10^{-8}$ \\
    $N_{\mathrm{ref}}$ & ground-truth HF solves (scrambled Sobol') & $2^{18}$ \\
\bottomrule
\end{tabular}
\end{table}

\subsubsection{Exact reduction to the ROM coordinates}

The full-class coefficient of Lemma~\ref{lem:optimal-full} cannot be applied
directly here. The LF solution $f^{(2)} = Vq$ lies in the $r$-dimensional POD
subspace, so $\Sigma_2$ has rank $r < d$ and is not invertible
(Assumption~\ref{ass:standing}(ii) fails). A ridge-regularized inverse would
only approximate the optimal coefficient, and it would still require
estimating $d\times d$ covariances from the pilot, which is what failed in
Experiment~1, Case~2. Both difficulties are avoided because $f^{(2)}$ is
determined by its $r$ reduced coordinates $q$, so a control variate built on
$f^{(2)}$ is equivalent to one built on $q$.

\begin{proposition}[Exact reduction for a rank-deficient LF covariance]
\label{prop:rom-reduction}
Let $f^{(2)} = Vq$, where $V\in\mathbb{R}^{d\times r}$ has orthonormal
columns ($V^\top V = I_r$) and the reduced coordinates $q\in\mathbb{R}^r$
have covariance $C_q := \Cov(q) \succ 0$, and let
$G := \Cov(f^{(1)},q) \in \mathbb{R}^{d\times r}$. Then every
$A\in\mathbb{R}^{d\times d}$ with $AV = GC_q^{-1}$ is an optimal full-class
coefficient; in particular, so is
\[
A^\star = GC_q^{-1}V^\top = \Sigma_{1,2}\Sigma_2^+,
\]
where $\Sigma_2^+$ is the Moore--Penrose pseudo-inverse of $\Sigma_2$. For
every such $A$,
\begin{equation}
\label{eq:reduced-correction}
A\bigl(\bar y^{(2)}_{m_2}-\bar y^{(2)}_{m_1}\bigr)
= GC_q^{-1}\bigl(\bar q_{m_2}-\bar q_{m_1}\bigr),
\qquad
\gfull{2} = \frac{\tr(MGC_q^{-1}G^\top)}{\tr(M\Sigma_1)}.
\end{equation}
\end{proposition}

\begin{proof}
Since $f^{(2)} = Vq$, we have $\Sigma_2 = VC_qV^\top$ and
$\Sigma_{2,1} = VG^\top$. Substituting these into the level-$2$
objective~\eqref{eq:cv-problem} shows that $\Phi_2$ depends on $A$ only
through $H := AV \in \mathbb{R}^{d\times r}$:
\[
\Phi_2(A) = \tr(MHC_qH^\top) - 2\tr(MHG^\top) =: \Psi(H),
\]
and every $H\in\mathbb{R}^{d\times r}$ is attained, e.g., by $A = HV^\top$.
The function $\Psi$ is the objective of Lemma~\ref{lem:optimal-full} with
$\Sigma_i$ and $\Sigma_{1,i}$ replaced by $C_q \succ 0$ and $G$, and
completing the square as in that proof gives
\[
\Psi(H)
= \tr\!\big(M(H-H^\star)\,C_q\,(H-H^\star)^\top\big)
- \tr(MGC_q^{-1}G^\top),
\qquad
H^\star := GC_q^{-1}.
\]
The first term is nonnegative and vanishes at $H = H^\star$. Hence the
minimum of $\Phi_2$ is $-\tr(MGC_q^{-1}G^\top)$, attained by every $A$ with
$AV = H^\star$, and Definition~\ref{def:effective-correlation} gives
$\gfull{2}$. For such $A$, the identity $\bar y^{(2)}_m = V\bar q_m$ gives
the correction in~\eqref{eq:reduced-correction}. Finally,
$\Sigma_2^+ = VC_q^{-1}V^\top$ because $V$ has orthonormal columns, so
$\Sigma_{1,2}\Sigma_2^+ = GV^\top VC_q^{-1}V^\top = A^\star$.
\end{proof}

In other words, the optimal correction regresses the HF output on the reduced
coordinates $q$ rather than on the reconstructed state $Vq$: it is the
coefficient of Lemma~\ref{lem:optimal-full} with $f^{(2)}$ replaced by $q$.
In practice, only $G\in\mathbb{R}^{d\times r}$ and
$C_q\in\mathbb{R}^{r\times r}$ are estimated from the pilot, and only an
$r\times r$ system is solved ($10\times10$ instead of $5329\times5329$), so
the full class needs $n_{\mathrm{pilot}} > r$ rather than
$n_{\mathrm{pilot}} > d$ (\Cref{tab:class-cost}).

\subsubsection{Case 1: full-field quantity of interest}

The QoI is the mean of the whole solution field,
$\mu \in \mathbb{R}^{5329}$, with $M = I$. The reference effective
correlations are
\begin{equation*}
  \gscalar{2} = 0.9919, \qquad
  \gdiag{2} = 0.9923, \qquad
  \gfull{2} = 0.99906,
\end{equation*}
ordered as guaranteed by
Corollary~\ref{corollary:Hierarchy of effective correlations}.

\begin{table}[htbp]
\caption{Experiment~2, Case~1 ($d = 5329$, field QoI, $M = I$): empirical MSE
$\pm$ one standard error over $N_{\mathrm{rep}} = 500$ repetitions, ratio to
HF-only MC, and the population-predicted
ratio~\eqref{eq:optimal-mse} including the pilot-cost factor
$p/p_{\mathrm{prod}}$. The two HF-only MC rows are separate runs of the same
estimator; their difference indicates the repetition noise.}
\label{tab:exp2case1}
\centering
\small
\input{tables/exp2_case1.tex}
\end{table}

\begin{figure}[htbp]
\centering
\includegraphics[width=\textwidth]{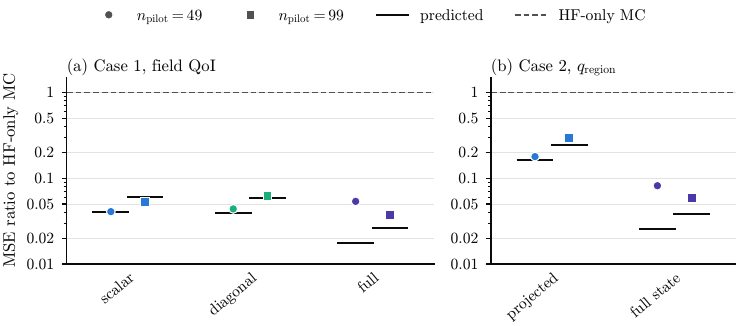}
\caption{Experiment~2: ratio of the empirical MSE of each MFMC estimator to
that of HF-only MC (log scale), with population-predicted ratios (black
bars), for $n_{\mathrm{pilot}} = 49$ and $n_{\mathrm{pilot}} = 99$.
(a)~Case~1, field QoI. (b)~Case~2, regional mean $q_{\mathrm{region}}$:
scalar MFMC on the projected outputs and the full-state control variate.}
\label{fig:exp2}
\end{figure}

Table~\ref{tab:exp2case1} and Figure~\ref{fig:exp2}(a) report the results.
Every MFMC estimator reduces the MSE by a factor of $16$--$27$ relative to
MC. The scalar and diagonal classes are close to their predictions already
at $n_{\mathrm{pilot}} = 49$, since they estimate only two traces or $2d$
componentwise covariances, which $49$ pilot samples determine accurately.

At $n_{\mathrm{pilot}} = 49$ the full class performs worse than both
componentwise classes, due to inaccurate coefficient estimation from insufficient samples.
Doubling the pilot ($n_{\mathrm{pilot}} = 99$) reverses the
ranking. The full class becomes the best estimator although its production
budget is a third smaller, while the scalar and diagonal classes
deteriorate, because their few parameters were already well estimated and
the extra pilot samples only consume production budget.

\subsubsection{Case 2: a mapped scalar quantity of interest}

This case illustrates Corollary~\ref{cor:full-state-benefit} for the scalar
linear mapping $q = c^\top \mu$ of the field mean given by the regional mean
temperature
\begin{equation*}
  q_{\mathrm{region}} = \frac{1}{|\mathcal{I}_R|}
    \sum_{j \in \mathcal{I}_R} \mu_j ,
\end{equation*}
where $\mathcal{I}_R$ indexes the $814$ grid points in the downstream region
$R = [0.5, 1] \times [0.6, 0.9]$. This quantity varies with all three
parameters, with a standard deviation of $12.6$\,K. We compare HF-only MC, scalar MFMC on the projected outputs
$g^{(i)} = c^\top f^{(i)}$, and the full-state
estimator~\eqref{eq:scalar-qoi-full-state-estimator} with coefficient
$c^\top \Sigma_{1,2}\Sigma_2^{+}$.

The projection loses part of the correlation:
$\rho^2(g^{(1)},g^{(2)}) = 0.928$, whereas the full-state control variate
attains $\gfull{2}|_{M = cc^\top} = 0.997$. Scalar MFMC must use the projected LF
output $c^\top V q$ as it is, whereas the full-state control variate
regresses the HF quantity on all $r$ reduced coordinates and thereby
corrects part of this error.

\begin{table}[htbp]
\caption{Experiment~2, Case~2 (regional mean $q_{\mathrm{region}}$, $M = cc^\top$): empirical MSE
$\pm$ one standard error over $N_{\mathrm{rep}} = 500$ repetitions, ratio to
HF-only MC, and the population-predicted ratio.
``Projected'' is scalar MFMC on $g^{(i)} = c^\top f^{(i)}$; ``full state'' is
the full-state control variate of Corollary~\ref{cor:full-state-benefit}.}
\label{tab:exp2case2}
\centering
\small
\input{tables/exp2_case2.tex}
\end{table}

Table~\ref{tab:exp2case2} and Figure~\ref{fig:exp2}(b) report the results.
Both MFMC estimators improve on HF-only MC. The full-state control variate has the lower MSE at both
pilot sizes, by a factor of $2.2$ ($n_{\mathrm{pilot}} = 49$) and $5.1$
($n_{\mathrm{pilot}} = 99$). This illustrates the benefit of correcting the
full state before projecting.

\subsection{Experiment 3: two-phase gas injection with a coarse-grid low-fidelity model}
\label{sec:exp3}

The third experiment uses a more challenging HF/LF model pair: a nonlinear, time-dependent
two-phase flow simulator and the same simulator on a coarser grid. Three
features make multifidelity estimation more challenging than in
Experiment~2. The coarse model is only about $5.4$ times cheaper (compared
with $140$ in Experiment~2). The output, a gas saturation field, has a sharp,
sample-dependent plume front that the coarse grid resolves only
approximately. Finally, many output components are zero in most samples, so
the pilot-estimated LF covariances are ill-conditioned.

\subsubsection{Model problem and the two fidelities}

We consider immiscible two-phase (water and gas) flow without capillary
pressure in a vertical cross-section of an aquifer. Mass conservation of
each phase $\alpha\in\{w,g\}$ combined with Darcy's law gives
\begin{equation}
  \frac{\partial}{\partial t}\!\left(\frac{\phi\, S_\alpha}{B_\alpha}\right)
  - \nabla\cdot\!\left(
      \frac{k\, k_{r\alpha}(S_w)}{\mu_\alpha B_\alpha}
      \bigl(\nabla p - \rho_\alpha g \nabla D\bigr)
    \right)
  = q_\alpha,
  \qquad S_w + S_g = 1,
  \label{eq:twophase}
\end{equation}
where $q_w = 0$ and $q_g$ is the gas injection rate. Water and rock are
slightly compressible, and the gas properties depend on pressure. The
aquifer is initially saturated with water, all boundaries are closed, and gas
is injected at a constant rate for $T = 900$\,days through a single well near
the bottom of the domain. The random input is the horizontal log-permeability, a
Gaussian random field conditioned on measurements at observation wells
and represented by its $437$ leading Karhunen--Lo\`eve modes, whose
coefficients form the random input $Z \sim \mathcal N(0, I_{437})$.

The HF model solves~\eqref{eq:twophase} by two-point-flux finite volumes on
a $30\times 30$ grid and returns the final gas saturation field,
$f^{(1)}(Z) = S_g(\cdot,T) \in \mathbb R^{d}$ with $d = 900$. The LF model
solves the same equations on a $15 \times 15$ grid, where the permeability
of each coarse cell is the geometric mean of the four fine cells it
contains. Its coarse final saturation $s(Z) \in \mathbb R^{225}$ is mapped
to the HF grid by assigning the value of each coarse cell to the four fine
cells it contains, $f^{(2)}(Z) = P\, s(Z)$ with $P \in \{0,1\}^{900\times 225}$.
The costs are $w_1 = 1$ and $w_2 = 0.185$ in HF-solve equivalents.
Figure~\ref{fig:exp3setup} shows one realization and the reference mean
field.

\begin{figure}[htbp]
\centering
\includegraphics[width=\textwidth]{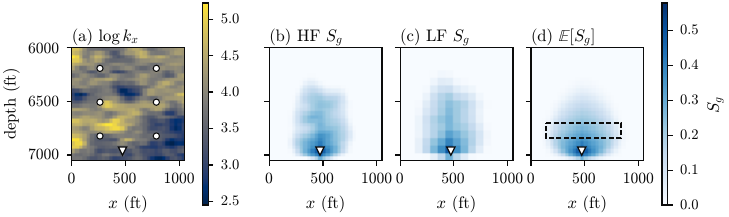}
\caption{Two-phase experiment. (a)~One log-permeability realization on the
$30 \times 30$ HF grid (circles: conditioning data; triangle: injection
well). (b), (c)~Final gas saturation of the HF ($30 \times 30$) and LF
($15 \times 15$) models for this realization. (d)~Reference mean of the HF
final gas saturation ($32{,}000$ samples); dashed: averaging region
$R_{\mathrm{band}}$ of Case~2.}
\label{fig:exp3setup}
\end{figure}

\subsubsection{Estimator implementation}

As in Experiment~2, the LF output lies in a known low-dimensional subspace,
here the range of $P$, so Proposition~\ref{prop:rom-reduction} applies (with
$V = P/2$) and the full class reduces to a regression of the HF output on
the $225$ coarse saturations. The block class uses $25$ blocks of
$6 \times 6$ HF cells, each regressing on the $9$ coarse cells it covers.
Coarse cells that gas never reaches in the pilot sample are zero in every
pilot sample, so their rows and columns of the pilot covariance vanish. The
regression gives them zero coefficients and uses only the remaining active coarse
cells.

The remaining coarse covariance is still ill-conditioned, because cells
near the plume front are reached by gas in only a few pilot samples and have
tiny but nonzero variances. The ridge $\lambda = 10^{-8}$ of Experiments~1
and~2 hardly regularizes these directions, whose regression coefficients
then amplify sampling noise. For the block class, the full class, and the
full-state control variate of Case~2, we therefore use a truncated
pseudo-inverse of the pilot covariance of $s$, which discards
eigendirections with eigenvalues below $10^{-2}$ times the largest. Results
with the ridge are reported for comparison. Table~\ref{tab:exp3params}
collects the parameters.

\begin{table}[htbp]
\caption{Parameters of the two-phase gas-injection experiment.}
\label{tab:exp3params}
\centering
\small
\begin{tabular}{@{}llc@{}}
\toprule
Parameter & Meaning & Value \\
\midrule
$n_x \times n_z$ & HF / LF grids & $30 \times 30$ / $15 \times 15$ \\
$d$ & HF output dimension (final $S_g$ field) & $900$ \\
$T$ & simulated time & $900$\,days \\
$w_1,\, w_2$ & nominal model costs (measured $5.9$\,s, $1.09$\,s) & $1,\ 0.185$ \\
$p$ & total budget (HF-solve equivalents) & $300$ \\
$n_{\mathrm{pilot}}$ & common pilot samples (nested) & $25$ / $50$ \\
$p_{\mathrm{prod}}$ & production budget per estimator & $270.4$ / $240.8$ \\
$m$ & HF-only MC samples (full budget) & $300$ \\
block partition & $6 \times 6$ HF tiles ($9$ LF coordinates each) & $25$ blocks \\
truncation & relative eigenvalue threshold for $\widehat{\Cov}(s)^{+}$ & $10^{-2}$ \\
$N_{\mathrm{rep}}$ & independent repetitions & $50$ \\
$N_{\mathrm{ref}}$ & reference HF samples (mean) / paired samples ($\gamma$) & $32{,}000$ / $10^4$ \\
\bottomrule
\end{tabular}
\end{table}

\subsubsection{Case 1: full saturation field}

The QoI is the mean of the whole final saturation field,
$\mu \in \mathbb R^{900}$, with $M = I$. The reference effective
correlations are
\begin{equation*}
  \gscalar{2} = 0.759, \qquad
  \gdiag{2} = 0.826, \qquad
  \gblock{2} = 0.907, \qquad
  \gfull{2} = 0.942,
\end{equation*}
ordered as stated in
Corollary~\ref{corollary:Hierarchy of effective correlations}. They are well
below the values $0.992$--$0.999$ of Experiment~2, because the coarse grid
resolves the sample-dependent plume front only approximately.
Table~\ref{tab:exp3case1} and Figure~\ref{fig:exp3}(a) report the results.

\begin{table}[htbp]
\caption{Experiment~3,
Case~1 ($d = 900$, field QoI, $M = I$): empirical MSE $\pm$ one standard
error over $N_{\mathrm{rep}} = 50$ repetitions, ratio to HF-only MC, and the
population-predicted ratio~\eqref{eq:optimal-mse}. The block and full classes use the truncated
pseudo-inverse; the rows marked ``ridge'' use the ridge $\lambda = 10^{-8}$
of Experiments~1 and~2 instead.}
\label{tab:exp3case1}
\centering
\small
\input{tables/exp3_case1.tex}
\end{table}

\begin{figure}[htbp]
\centering
\includegraphics[width=\textwidth]{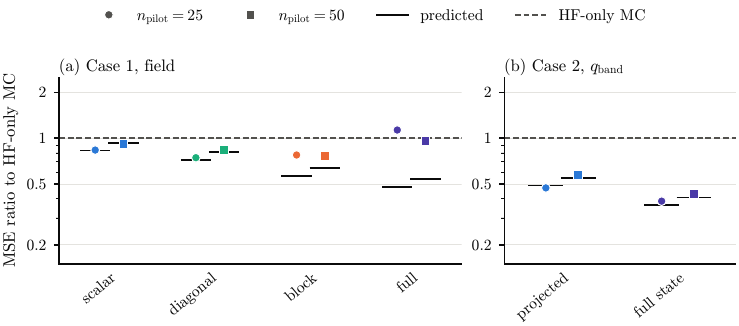}
\caption{Experiment~3: ratio of the empirical MSE of each MFMC estimator to
that of HF-only MC (log scale), with population-predicted ratios (black
bars), for
$n_{\mathrm{pilot}} = 25$ and $50$. (a)~Case~1, full saturation field.
(b)~Case~2, band-mean saturation $q_{\mathrm{band}}$ (projected scalar
MFMC and full-state control variate). The block, full, and full-state estimators use the
truncated pseudo-inverse.}
\label{fig:exp3}
\end{figure}

At $n_{\mathrm{pilot}} = 25$ the scalar and diagonal classes reduce the MSE
by $16\%$ and $26\%$. Their empirical ratios agree with the predictions, so
formula~\eqref{eq:optimal-mse} remains accurate for this nonlinear model
pair when only a few covariance quantities are estimated.

The block and full classes show why the truncated pseudo-inverse is needed.
With truncation, the block class has the lowest MSE at
$n_{\mathrm{pilot}} = 50$, since each block regresses on only $9$ coarse
cells. The full class remains the weakest MFMC estimator and is even worse
than HF-only MC at $n_{\mathrm{pilot}} = 25$. This is because its regression on about $156$
active coarse cells is estimated from only $25$--$50$ pilot samples, so even
after truncation to $13$--$15$ directions its coefficients carry substantial
sampling noise. With the ridge instead, both classes fail. Within blocks at
the plume edge, rarely invaded cells make the $9\times 9$ block covariances
nearly singular, and the block class is $18$ times worse than HF-only MC at
$n_{\mathrm{pilot}} = 25$. The full class fails as in Experiment~1, Case~2:
its pilot covariance has rank at most $n_{\mathrm{pilot}}-1$, below the
number of active cells, so $\hat\gamma \approx 1$ and the allocation assigns
$m_1 = 1$ in every repetition.

\subsubsection{Case 2: mean gas saturation in a band across the plume}

As in Experiment~2, Case~2, we compare HF-only MC, scalar MFMC on the
projected outputs, and the full-state control variate, here a regression of
$g^{(1)}$ on the LF coarse coordinates of the region. The QoI is the mean
final gas saturation over the horizontal band
$R_{\mathrm{band}} = [140, 840] \times [6700, 6840]$\,ft (dashed in
Figure~\ref{fig:exp3setup}(d)),
\begin{equation*}
  q_{\mathrm{band}} = \frac{1}{|\mathcal{I}_R|}
    \sum_{j \in \mathcal{I}_R} \mu_j ,
\end{equation*}
where $\mathcal{I}_R$ indexes the $80$ HF cells in $R_{\mathrm{band}}$,
which cover $20$ coarse cells.

The projection loses part of the correlation:
$\rho^2(g^{(1)}, g^{(2)}) = 0.938$, whereas the full-state control variate
attains $\gfull{2}\big|_{M = cc^\top} = 0.978$. Table~\ref{tab:exp3case2} and Figure~\ref{fig:exp3}(b) 
report the detailed results.

\begin{table}[htbp]
\caption{Experiment~3,
Case~2 (band-mean gas saturation $q_{\mathrm{band}}$, $M = cc^\top$):
empirical MSE $\pm$ one standard error over $N_{\mathrm{rep}} = 50$
repetitions, ratio to HF-only MC, and the population-predicted ratio.
The full-state control variate uses the truncated pseudo-inverse; the rows
marked ``ridge'' use the ridge $\lambda = 10^{-8}$.}
\label{tab:exp3case2}
\centering
\small
\input{tables/exp3_case2.tex}
\end{table}

Both MFMC estimators improve on HF-only MC. The full-state control variate has the
lower MSE at both pilot sizes, by factors of $1.22$ and $1.34$ against a
predicted $1.35$. Together with
Experiment~2, this case shows that correcting first helps when the
projection discards correlation that the LF state still carries, also for
an LF model that differs from the HF model in grid resolution rather than
by a reduced basis.

\subsection{Pilot reuse}
\label{sec:exp-reuse}

We recomputed the estimators of Sections~\ref{sec:experiment-gaussian},
\ref{sec:exp2}, and~\ref{sec:exp3} with the pilot samples
reused in the final estimator as described in Section~\ref{sec:pilot-reuse},
keeping all designs unchanged. Both estimators then share the
pilot, the design, and the production samples in each repetition, and
Figure~\ref{fig:reuse} reports the ratio of their empirical MSEs.

\begin{figure}[htbp]
\centering
\includegraphics[width=\textwidth]{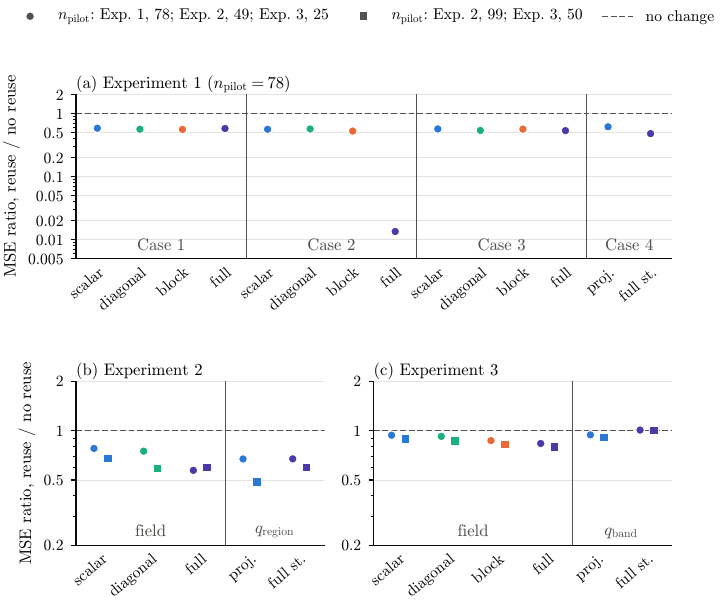}
\caption{Pilot reuse: ratio of the empirical MSE with the pilot samples
reused in the final estimator (Section~\ref{sec:pilot-reuse}) to that of the
estimator of Algorithm~\ref{alg:pilot_reduced_mfmc} with the same design (log scale).
(a)~Experiment~1.
(b)~Experiment~2. (c)~Experiment~3, Cases~1 and~2. Values below one mean
that reuse lowers the MSE.}
\label{fig:reuse}
\end{figure}

Reuse lowers the empirical MSE in $34$ of the $36$ comparisons. The two
exceptions, the full-state control variate of the Experiment~3 band mean,
rise by less than $1\%$, which is within $0.05$ standard errors. The
reduction grows with
the share of the budget spent on the pilot: $38$--$52\%$ in Experiment~1
(excluding the failed full class of Case~2), where the pilot costs half of
the budget, $22$--$51\%$ in Experiment~2, where it costs a quarter to a
half, and $6$--$21\%$ for the Experiment~3 field, where it costs
$10$--$20\%$. Reuse also removes the failure of the full class in
Experiment~1, Case~2, whose HF sample mean now uses $79$ samples instead of
one. In these experiments, the variance reduction from the additional
samples therefore outweighed the bias introduced by reusing the pilot
samples.

\section{Conclusion}
\label{sec:conclusion}

We have extended multifidelity Monte Carlo (MFMC) to vector-valued
quantities of interest by replacing the scalar control-variate coefficients
with matrices and minimizing a weighted mean-squared error. For the full
matrix class, the optimal coefficient is
$A_i^\star = \Sigma_{1,i}\Sigma_i^{-1}$, and its variance reduction is
summarized by an effective correlation $\gamma_{1,i}\in[0,1]$. With
$\rho_{1,i}^2$ replaced by $\gamma_{1,i}$, the closed-form sample
allocation, the cost--correlation ordering condition, and the model
selection of scalar MFMC carry over unchanged
(Theorem~\ref{thm:main-vvmfmc}), and the scalar theory is recovered for
$d = 1$. The block-diagonal, diagonal, and scalar classes also have
closed-form coefficients and effective correlations for compatible weighting
matrices, and they form a nested hierarchy along which the population
variance reduction does not decrease
(Corollary~\ref{corollary:Hierarchy of effective correlations}).

In practice, the choice of class is a tradeoff between population
optimality and the stability of the pilot estimates. In the experiments, on
a Gaussian problem, a POD--Galerkin ROM of an advection--diffusion problem,
and a coarse-grid model of two-phase gas injection, the realized ranking of
the classes is governed by the pilot size relative to the number of
covariance quantities a class must estimate. The full class realizes its
advantage only when the pilot is large enough; otherwise it loses its
advantage (Experiment~2 at $n_{\mathrm{pilot}} = 49$, Experiment~3) or
becomes unstable (Experiment~1, Case~2), whereas the reduced classes remain
robust. The class should therefore be matched to the pilot budget (e.g.,
block size below $n_{\mathrm{pilot}}$) and to the correlation structure of
the output, and the pilot fraction worth spending grows with the
expressiveness of the class. When the LF model is a ROM, the full class
reduces exactly to a regression on the $r$ reduced coordinates
(Proposition~\ref{prop:rom-reduction}), so its pilot demand scales with $r$
rather than with $d$. Reusing the pilot samples in the final estimator
introduces a bias, but it lowered the MSE in $34$ of the $36$ comparisons.

Finally, for a scalar quantity of interest that is a linear projection of
the model state, correcting the full state and then projecting is, with
accurate covariance information, at least as accurate as scalar MFMC on the
projected outputs (Corollary~\ref{cor:full-state-benefit}). The gain is small when the
projected outputs are already highly correlated, and it can be large when the
projection discards part of the correlation: for the regional mean
temperature in Experiment~2 ($\rho^2 = 0.93$ against an effective correlation
of $0.997$), the full-state control variate reduced the MSE of scalar MFMC by
a factor of $2.2$--$5.1$, and for the band-mean gas saturation in
Experiment~3 ($0.94$ against $0.98$) by a factor of $1.2$--$1.3$. Its
regression on the LF state is estimated from the pilot and needs the same
regularization of the LF covariance as the full class.

\section*{Acknowledgments}

This research was supported in part by the Air Force Office of Scientific Research under grants FA9550-21-1-0381 and FA9550-24-1-0237, and by the Office of Advanced Scientific Computing Research (ASCR) within the Department of Energy Office of Science under award number DE-SC0023163.

\bibliographystyle{siamplain}
\bibliography{mybib}
\end{document}

%% file: tables/exp1.tex
\begin{tabular}{@{}lrrr@{}}
\toprule
Method & MSE $\pm$ SE & Ratio & Predicted ratio \\
\midrule
\multicolumn{4}{@{}l}{Case 1: $d = 10$, $k = 2$} \\
\addlinespace[2pt]
MC (HF only) & $0.1760 \pm 0.0093$ & 1 & 1 \\
MFMC, scalar & $0.0676 \pm 0.0023$ & 0.38 & 0.38 \\
MFMC, diagonal & $0.0643 \pm 0.0020$ & 0.37 & 0.37 \\
MFMC, block & $0.0536 \pm 0.0020$ & 0.30 & 0.29 \\
MFMC, full & $0.0518 \pm 0.0021$ & 0.29 & 0.28 \\
\midrule
\multicolumn{4}{@{}l}{Case 2: $d = 100$, $k = 2$} \\
\addlinespace[2pt]
MC (HF only) & $1.707 \pm 0.050$ & 1 & 1 \\
MFMC, scalar & $0.670 \pm 0.012$ & 0.39 & 0.38 \\
MFMC, diagonal & $0.641 \pm 0.011$ & 0.38 & 0.37 \\
MFMC, block & $0.611 \pm 0.012$ & 0.36 & 0.29 \\
MFMC, full & $75.1 \pm 1.7$ & 44.0 & 0.28 \\
\midrule
\multicolumn{4}{@{}l}{Case 3: $d = 10$, $k = 3$} \\
\addlinespace[2pt]
MC (HF only) & $0.1797 \pm 0.0096$ & 1 & 1 \\
MFMC, scalar & $0.0666 \pm 0.0022$ & 0.37 & 0.36 \\
MFMC, diagonal & $0.0605 \pm 0.0020$ & 0.34 & 0.35 \\
MFMC, block & $0.0475 \pm 0.0018$ & 0.26 & 0.25 \\
MFMC, full & $0.0502 \pm 0.0019$ & 0.28 & 0.24 \\
\midrule
\multicolumn{4}{@{}l}{Case 4: scalar QoI $q = c^\top\mu$, $M = cc^\top$} \\
\addlinespace[2pt]
MC (HF only) & $(1.284 \pm 0.073)\times 10^{-2}$ & 1 & 1 \\
MFMC, projected & $(1.146 \pm 0.081)\times 10^{-2}$ & 0.89 & 0.85 \\
MFMC, full state & $(3.98 \pm 0.25)\times 10^{-3}$ & 0.31 & 0.22 \\
\bottomrule
\end{tabular}

%% file: tables/exp2_case1.tex
\begin{tabular}{@{}lrrr@{}}
\toprule
Method & MSE $\pm$ SE & Ratio & Predicted ratio \\
\midrule
\multicolumn{4}{@{}l}{$n_{\mathrm{pilot}} = 49$ ($p_{\mathrm{prod}} = 150$)} \\
\addlinespace[2pt]
MC (HF only) & $(6.86 \pm 0.40)\times 10^{4}$ & 1 & 1 \\
MFMC, scalar & $(2.80 \pm 0.12)\times 10^{3}$ & 0.041 & 0.040 \\
MFMC, diagonal & $(3.01 \pm 0.13)\times 10^{3}$ & 0.044 & 0.039 \\
MFMC, full & $(3.69 \pm 0.28)\times 10^{3}$ & 0.054 & 0.018 \\
\midrule
\multicolumn{4}{@{}l}{$n_{\mathrm{pilot}} = 99$ ($p_{\mathrm{prod}} = 100$)} \\
\addlinespace[2pt]
MC (HF only) & $(7.33 \pm 0.44)\times 10^{4}$ & 1 & 1 \\
MFMC, scalar & $(3.89 \pm 0.16)\times 10^{3}$ & 0.053 & 0.061 \\
MFMC, diagonal & $(4.55 \pm 0.22)\times 10^{3}$ & 0.062 & 0.059 \\
MFMC, full & $(2.73 \pm 0.13)\times 10^{3}$ & 0.037 & 0.027 \\
\bottomrule
\end{tabular}

%% file: tables/exp2_case2.tex
\begin{tabular}{@{}lrrr@{}}
\toprule
Method & MSE $\pm$ SE & Ratio & Predicted ratio \\
\midrule
\multicolumn{4}{@{}l}{$q_{\mathrm{region}}$, $n_{\mathrm{pilot}} = 49$} \\
\addlinespace[2pt]
MC (HF only) & $0.878 \pm 0.054$ & 1 & 1 \\
MFMC, projected & $0.157 \pm 0.010$ & 0.18 & 0.16 \\
MFMC, full state & $0.072 \pm 0.013$ & 0.082 & 0.025 \\
\midrule
\multicolumn{4}{@{}l}{$q_{\mathrm{region}}$, $n_{\mathrm{pilot}} = 99$} \\
\addlinespace[2pt]
MC (HF only) & $0.722 \pm 0.042$ & 1 & 1 \\
MFMC, projected & $0.213 \pm 0.016$ & 0.29 & 0.24 \\
MFMC, full state & $0.0422 \pm 0.0028$ & 0.058 & 0.038 \\
\bottomrule
\end{tabular}

%% file: tables/exp3_case1.tex
\begin{tabular}{@{}lrrr@{}}
\toprule
Method & MSE $\pm$ SE & Ratio & Predicted ratio \\
\midrule
\multicolumn{4}{@{}l}{$n_{\mathrm{pilot}} = 25$} \\
\addlinespace[2pt]
MC (HF only) & $(1.52 \pm 0.11)\times 10^{-3}$ & 1 & 1 \\
MFMC, scalar & $(1.273 \pm 0.079)\times 10^{-3}$ & 0.84 & 0.83 \\
MFMC, diagonal & $(1.135 \pm 0.061)\times 10^{-3}$ & 0.74 & 0.72 \\
MFMC, block & $(1.182 \pm 0.065)\times 10^{-3}$ & 0.78 & 0.57 \\
MFMC, full & $(1.722 \pm 0.088)\times 10^{-3}$ & 1.13 & 0.48 \\
MFMC, block (ridge) & $(2.80 \pm 0.61)\times 10^{-2}$ & 18.4 & 0.57 \\
MFMC, full (ridge) & $(1.305 \pm 0.072)\times 10^{-1}$ & 85.6 & 0.48 \\
\midrule
\multicolumn{4}{@{}l}{$n_{\mathrm{pilot}} = 50$} \\
\addlinespace[2pt]
MC (HF only) & $(1.52 \pm 0.11)\times 10^{-3}$ & 1 & 1 \\
MFMC, scalar & $(1.401 \pm 0.092)\times 10^{-3}$ & 0.92 & 0.93 \\
MFMC, diagonal & $(1.276 \pm 0.080)\times 10^{-3}$ & 0.84 & 0.81 \\
MFMC, block & $(1.160 \pm 0.057)\times 10^{-3}$ & 0.76 & 0.64 \\
MFMC, full & $(1.459 \pm 0.066)\times 10^{-3}$ & 0.96 & 0.54 \\
MFMC, block (ridge) & $(4.46 \pm 0.65)\times 10^{-3}$ & 2.93 & 0.64 \\
MFMC, full (ridge) & $(1.271 \pm 0.096)\times 10^{-1}$ & 83.4 & 0.54 \\
\bottomrule
\end{tabular}

%% file: tables/exp3_case2.tex
\begin{tabular}{@{}lrrr@{}}
\toprule
Method & MSE $\pm$ SE & Ratio & Predicted ratio \\
\midrule
\multicolumn{4}{@{}l}{$n_{\mathrm{pilot}} = 25$} \\
\addlinespace[2pt]
MC (HF only) & $(6.79 \pm 0.94)\times 10^{-7}$ & 1 & 1 \\
MFMC, projected & $(3.20 \pm 0.63)\times 10^{-7}$ & 0.47 & 0.49 \\
MFMC, full state & $(2.62 \pm 0.46)\times 10^{-7}$ & 0.39 & 0.36 \\
MFMC, full state (ridge) & $(9.6 \pm 2.6)\times 10^{-7}$ & 1.41 & 0.36 \\
\midrule
\multicolumn{4}{@{}l}{$n_{\mathrm{pilot}} = 50$} \\
\addlinespace[2pt]
MC (HF only) & $(6.79 \pm 0.94)\times 10^{-7}$ & 1 & 1 \\
MFMC, projected & $(3.89 \pm 0.70)\times 10^{-7}$ & 0.57 & 0.55 \\
MFMC, full state & $(2.91 \pm 0.55)\times 10^{-7}$ & 0.43 & 0.41 \\
MFMC, full state (ridge) & $(4.19 \pm 0.97)\times 10^{-7}$ & 0.62 & 0.41 \\
\bottomrule
\end{tabular}

%% file: main.bbl
\begin{thebibliography}{10}

\bibitem{croci2023multi}
{\sc M.~Croci, K.~E. Willcox, and S.~J. Wright}, {\em Multi-output multilevel
  best linear unbiased estimators via semidefinite programming}, Computer
  Methods in Applied Mechanics and Engineering, 413 (2023), p.~116130.

\bibitem{destouches2023multivariate}
{\sc M.~Destouches, P.~Mycek, and S.~G{\"u}rol}, {\em Multivariate extensions
  of the multilevel best linear unbiased estimator for ensemble-variational
  data assimilation}, arXiv preprint arXiv:2306.07017,  (2023).

\bibitem{dixon2024covariance}
{\sc T.~O. Dixon, J.~E. Warner, G.~F. Bomarito, and A.~A. Gorodetsky}, {\em
  Covariance expressions for multifidelity sampling with multioutput,
  multistatistic estimators: Application to approximate control variates},
  SIAM/ASA Journal on Uncertainty Quantification, 12 (2024), pp.~1005--1049.

\bibitem{farcas2023context}
{\sc I.-G. Farca{\c{s}}, B.~Peherstorfer, T.~Neckel, F.~Jenko, and H.-J.
  Bungartz}, {\em Context-aware learning of hierarchies of low-fidelity models
  for multi-fidelity uncertainty quantification}, Computer Methods in Applied
  Mechanics and Engineering, 406 (2023), p.~115908.

\bibitem{fleeter2020multilevel}
{\sc C.~M. Fleeter, G.~Geraci, D.~E. Schiavazzi, A.~M. Kahn, and A.~L.
  Marsden}, {\em Multilevel and multifidelity uncertainty quantification for
  cardiovascular hemodynamics}, Computer methods in applied mechanics and
  engineering, 365 (2020), p.~113030.

\bibitem{geraci2015multifidelity}
{\sc G.~Geraci, M.~Eldred, and G.~Iaccarino}, {\em A multifidelity control
  variate approach for the multilevel {M}onte {C}arlo technique}, Center for
  Turbulence Research Annual Research Briefs,  (2015), pp.~169--181.

\bibitem{geraci2017multifidelity}
{\sc G.~Geraci, M.~S. Eldred, and G.~Iaccarino}, {\em A multifidelity
  multilevel {M}onte {C}arlo method for uncertainty propagation in aerospace
  applications}, in 19th AIAA non-deterministic approaches conference, 2017,
  p.~1951.

\bibitem{giles2008multilevel}
{\sc M.~B. Giles}, {\em Multilevel {M}onte {C}arlo path simulation}, Operations
  research, 56 (2008), pp.~607--617.

\bibitem{giles2015multilevel}
{\sc M.~B. Giles}, {\em Multilevel {M}onte {C}arlo methods}, Acta Numerica, 24
  (2015), pp.~259--328, \url{https://doi.org/10.1017/S096249291500001X}.

\bibitem{gorodetsky2020generalized}
{\sc A.~A. Gorodetsky, G.~Geraci, M.~S. Eldred, and J.~D. Jakeman}, {\em A
  generalized approximate control variate framework for multifidelity
  uncertainty quantification}, Journal of Computational Physics, 408 (2020),
  p.~109257.

\bibitem{gruber2026advancements}
{\sc A.~Gruber, G.~Geraci, P.~J. Blonigan, E.~Parish, and M.~Eldred}, {\em
  Advancements in reduced order model-based multifidelity uncertainty
  quantification for high-speed flows}, in AIAA SCITECH 2026 Forum, 2026,
  p.~1614.

\bibitem{jabarullah2019machine}
{\sc N.~K. Jabarullah~Khan and A.~H. Elsheikh}, {\em A machine learning based
  hybrid multi-fidelity multi-level {M}onte {C}arlo method for uncertainty
  quantification}, Frontiers in Environmental Science, 7 (2019), p.~105.

\bibitem{konrad2022data}
{\sc J.~Konrad, I.-G. Farca{\c{s}}, B.~Peherstorfer, A.~Di~Siena, F.~Jenko,
  T.~Neckel, and H.-J. Bungartz}, {\em Data-driven low-fidelity models for
  multi-fidelity {M}onte {C}arlo sampling in plasma micro-turbulence analysis},
  Journal of Computational Physics, 451 (2022), p.~110898.

\bibitem{lavenberg1982statistical}
{\sc S.~S. Lavenberg, T.~L. Moeller, and P.~D. Welch}, {\em Statistical results
  on control variables with application to queueing network simulation},
  Operations Research, 30 (1982), pp.~182--202.

\bibitem{ng2014multifidelity}
{\sc L.~W. Ng and K.~E. Willcox}, {\em Multifidelity approaches for
  optimization under uncertainty}, International Journal for numerical methods
  in Engineering, 100 (2014), pp.~746--772.

\bibitem{peherstorfer2019multifidelity}
{\sc B.~Peherstorfer}, {\em Multifidelity {M}onte {C}arlo estimation with
  adaptive low-fidelity models}, SIAM/ASA Journal on Uncertainty
  Quantification, 7 (2019), pp.~579--603.

\bibitem{peherstorfer2018convergence}
{\sc B.~Peherstorfer, M.~Gunzburger, and K.~Willcox}, {\em Convergence analysis
  of multifidelity {M}onte {C}arlo estimation}, Numerische Mathematik, 139
  (2018), pp.~683--707.

\bibitem{peherstorfer2016optimal}
{\sc B.~Peherstorfer, K.~Willcox, and M.~Gunzburger}, {\em Optimal model
  management for multifidelity {M}onte {C}arlo estimation}, SIAM Journal on
  Scientific Computing, 38 (2016), pp.~A3163--A3194.

\bibitem{peherstorfer2018survey}
{\sc B.~Peherstorfer, K.~Willcox, and M.~Gunzburger}, {\em Survey of
  multifidelity methods in uncertainty propagation, inference, and
  optimization}, Siam Review, 60 (2018), pp.~550--591.

\bibitem{pham2022ensemble}
{\sc T.~Pham and A.~A. Gorodetsky}, {\em Ensemble approximate control variate
  estimators: Applications to multifidelity importance sampling}, SIAM/ASA
  Journal on Uncertainty Quantification, 10 (2022), pp.~1250--1292.

\bibitem{qian2024multifidelity}
{\sc E.~Qian, D.~Kang, V.~Sella, and A.~Chaudhuri}, {\em Multifidelity linear
  regression for scientific machine learning from scarce data}, arXiv preprint
  arXiv:2403.08627,  (2024).

\bibitem{qian2018multifidelity}
{\sc E.~Qian, B.~Peherstorfer, D.~O'Malley, V.~V. Vesselinov, and K.~Willcox},
  {\em Multifidelity {M}onte {C}arlo estimation of variance and sensitivity
  indices}, SIAM/ASA Journal on Uncertainty Quantification, 6 (2018),
  pp.~683--706.

\bibitem{quaglino2019high}
{\sc A.~Quaglino, S.~Pezzuto, and R.~Krause}, {\em High-dimensional and
  higher-order multifidelity {M}onte {C}arlo estimators}, Journal of
  Computational Physics, 388 (2019), pp.~300--315.

\bibitem{rubinstein1985efficiency}
{\sc R.~Y. Rubinstein and R.~Marcus}, {\em Efficiency of multivariate control
  variates in {M}onte {C}arlo simulation}, Operations Research, 33 (1985),
  pp.~661--677.

\bibitem{schaden2020multilevel}
{\sc D.~Schaden and E.~Ullmann}, {\em On multilevel best linear unbiased
  estimators}, SIAM/ASA Journal on Uncertainty Quantification, 8 (2020),
  pp.~601--635.

\bibitem{stanley2026sample}
{\sc M.~Stanley, T.~Coons, G.~Bomarito, P.~Leser, J.~Pribe, and J.~Warner},
  {\em Sample correlation adjustments for robust multi-fidelity {M}onte {C}arlo
  under limited pilot sampling}, arXiv preprint arXiv:2605.23016,  (2026).

\bibitem{venkatraman1986efficiency}
{\sc S.~Venkatraman and J.~R. Wilson}, {\em The efficiency of control variates
  in multiresponse simulation}, Operations Research Letters, 5 (1986),
  pp.~37--42.

\bibitem{zanoni2024improved}
{\sc A.~Zanoni, G.~Geraci, M.~Salvador, K.~Menon, A.~L. Marsden, and D.~E.
  Schiavazzi}, {\em Improved multifidelity {M}onte {C}arlo estimators based on
  normalizing flows and dimensionality reduction techniques}, Computer methods
  in applied mechanics and engineering, 429 (2024), p.~117119.

\end{thebibliography}
